\documentclass{article}

\usepackage{arxiv}
\usepackage[utf8]{inputenc} 
\usepackage[T1]{fontenc}    
\usepackage{hyperref}       
\usepackage{url}            
\usepackage{booktabs}       
\usepackage{amsfonts}       
\usepackage{nicefrac}       
\usepackage{microtype}      
\usepackage{lipsum}		
\usepackage{graphicx}
\usepackage{natbib}
\usepackage{doi}
\usepackage{wrapfig}
\usepackage{algorithm}
\usepackage[noend]{algpseudocode}
\usepackage[bottom]{footmisc}
\newtheorem{theorem}{Theorem}
\newtheorem{lemma}{Lemma}
\newtheorem{corollary}{Corollary}

\newtheorem{example}{Example}
\newenvironment{proof}{\paragraph{Proof.}}{\hfill$\square$}

\title{Spiking Neural Networks in the Alexiewicz Topology: 
A New Perspective on Analysis and Error Bounds}

\author{ 
{
\hspace{1mm}Bernhard A.~Moser}\thanks{double affiliation: Software Competence Center Hagenberg (SCCH), 4232 Hagenberg, Austria} \\
	Institute of Signal Processing \\
	Johannes Kepler University of Linz\\
	\texttt{bernhard.moser@\{scch.at,jku.at\}} 
	\And
	{\hspace{1mm}Michael Lunglmayr} \\
	Institute of Signal Processing\\
	Johannes Kepler University of Linz, Austria\\
	\texttt{michael.lunglmayr@jku.at} 
	}

\renewcommand{\shorttitle}{\textit{arXiv} SNNs in the Alexiewicz Topology}

\hypersetup{
pdftitle={A template for the arxiv style},
pdfsubject={q-bio.NC, q-bio.QM},
pdfauthor={Bernhard A.~Moser, Michael Lunglmayr},
pdfkeywords={Neuromorphic Computing, Spiking Neuron Model, Threshold-based Sampling},
}

\begin{document}
\maketitle

\begin{abstract}
In order to ease the analysis of error propagation in neuromorphic computing and to get a better understanding of spiking neural networks 
(SNN), we address the problem of mathematical analysis of SNNs as endomorphisms that map spike trains to spike trains. 
A central question is the adequate structure for a space of spike trains and its implication for the design of error measurements of SNNs including time delay, threshold deviations, and the design of the reinitialization mode of the leaky-integrate-and-fire (LIF) neuron model.
First we identify the underlying topology by analyzing the closure of all sub-threshold signals of a LIF model. 
For zero leakage this approach yields the Alexiewicz topology, which we adopt to LIF neurons with arbitrary positive leakage.
As a result LIF can be understood as spike train quantization in the corresponding norm. 
This way we obtain various error bounds and inequalities such as a quasi isometry relation between incoming and outgoing spike trains. 
Another result is a Lipschitz-style global upper  bound for the error propagation and a related resonance-type phenomenon.
\end{abstract}


\section{Introduction}
\label{sec:Introduction}
Spiking neural networks (SNNs) are artificial neural networks 
of interconnected neurons that asynchronously  process and transmit spatial-temporal information based on the occurrence of spikes that come  from spatially distributed sensory input neurons. The most commonly used neuron model in SNNs is the leaky-integrate-and-fire (LIF) model~\cite{bookGerstner2014}. 
Despite its strong simplification of the biological way of spike generation, the LIF  model has proven useful in particular when 
modeling the temporal spiking characteristics in the biological nervous system.  
For an overview see, e.g.,~\cite{TAVANAEI2019, Nunes2022}.
At the interfaces, that is from real-world to SNN, and from SNN output to real-world, in general there is the need to translate analogue perceived intensities into spike trains, and later on, after processing, to decode the resulting spike trains into meaningful decisions.
Some approaches prefer a rate-based encoding while others emphasize on  the timing, e.g., of first arriving spikes~\cite{Guo2021}.
In the context of sampling time-varying signals also alternatives of LIF are encountered to sample 
real analogue signals into spikes, e.g., by delta coding or synonymously used terms like send-on-delta, level crossing or threshold-based sampling~\cite{miskowicz2006sendondelta,Tobi2014,Yousefzadeh2022}.

Due to their particular nature of asynchronous and sparse information processing, SNNs are studied mainly for two reasons: first, as a simplified mathematical model in the context of computational neuroscience aiming at a better understanding of biological neural circuits, and second, as a further  step  towards more powerful but energy-efficient embedded AI edge solutions to process time-varying signals 
 with a wide range of applications including visual processing~\cite{Amir2017,Yousefzadeh2022}, audio recognition~\cite{SNNSound2018}, speech recognition~\cite{SpeechRec2020}, biomedical signal processing~\cite{Hassan2018RealTimeCA} and robotic control~\cite{Kabilan2021,Yamazaki2022}. New application scenarios are emerging in the context of edge AI and federated learning across a physically distributed network of resource-constrained edge devices to collaboratively train a global model while preserving privacy~\cite{YangFedSNN2022}. 
Of particular interest are applications in the emerging field of brain-computer interfaces which  opens up new perspectives for the treatment of neurological diseases such as Parkinson's disease~\cite{Dethier_2013, Gege2021}.
For an overview on the performance comparison between SNNs and conventional vector-based artificial neural networks see~\cite{DENG2020294}.
However, the full potential of SNNs, in particular their energy efficiency and dynamic properties, will only become manifest when implemented on dedicated neuromorphic hardware, leading to ongoing research in this direction~\cite{Bouvier2019SpikingNN,TrueNorth2019,Ostrau2022,Michaelis2022}.

However, despite the great potential of SNNs and the ongoing research efforts the practical realizations are few so far. 
To this end, the research on SNNs is still in an early phase of maturity, particularly
lacking mathematical foundation which becomes necessary due to the special hybrid continuous-discrete nature of SNNs and the underlying 
paradigm shift towards event-based signal processing. 

A closer look at the different ways of sampling makes apparent this paradigm shift as sketched in Fig.~\ref{fig:ParadigmShift}. In equidistant-based uniform sampling, the mathematics of information encoding and processing is based on regular Dirac combs and its embedding in Hilbert spaces with its powerful mathematical tools of convolution, orthogonal projection, and based thereupon spectral analysis, signal filtering and reconstruction. 
As sequences of uniformly distributed Dirac pulses in time, regular Dirac combs and related concepts of signal decomposition into regular wave forms can be viewed as a trick that allows time to be treated mathematically as a space variable. 
However, this mathematical abstraction neglects the time information that is implicitly encoded by events~\cite{miskowicz2006sendondelta}.
To this end, the resulting mathematics of uniform sampling and signal processing becomes basically vector-based.  
In contrast, in biological information processing systems and bio-inspired neuromorphic computing, see e.g.,~\cite{TAVANAEI2019},
~\cite{Nunes2022}, the paradigm of information encoding somewhat flips the role of regularity w.r.t time versus amplitude.
While in uniform sampling and related signal processing time is treated as a regular structure and the amplitudes of sampled values
are kept flexible, in bio-inspired sampling and signal processing the amplitudes are forced into a regular structure by means of thresholding while keeping the handling of time flexible. 
\begin{wrapfigure}{r}{10cm}
  \includegraphics[width=9.5cm]{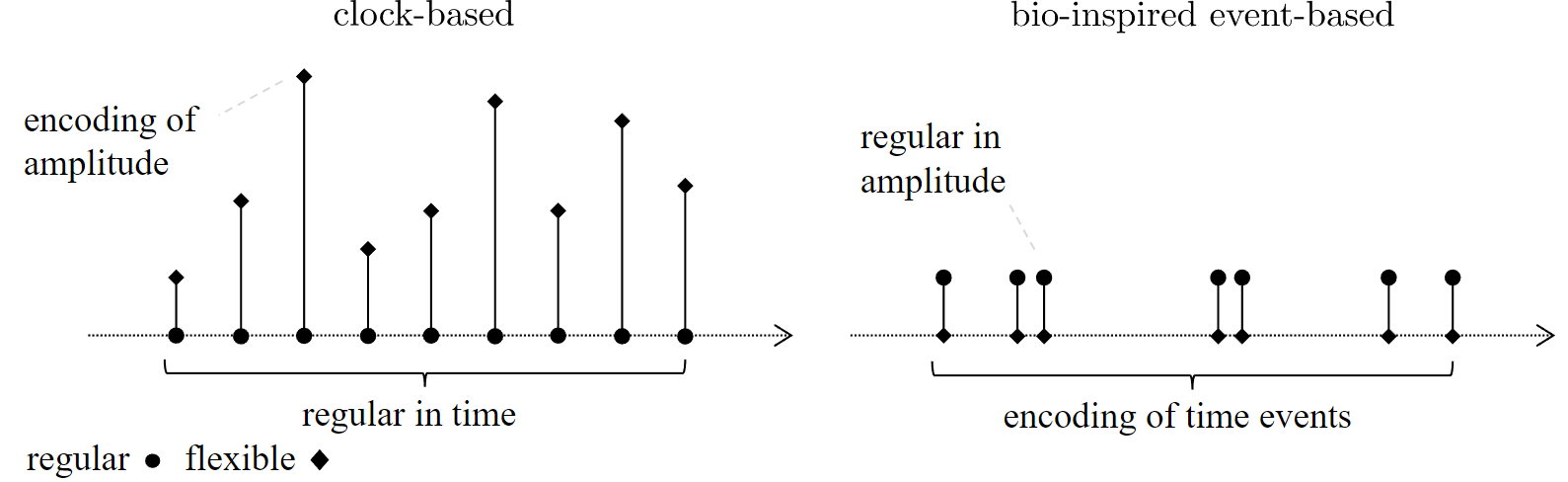}
  \caption{Paradigm shift in information encoding of uniform (left) versus threshold-based (right) sampling.}
		\label{fig:ParadigmShift}
\end{wrapfigure}
This paradigm shift of information encoding will have implications on the mathematical foundation of handling 
irregularity in time, that is to handle sequences of Dirac impulses beyond a regular comb structure, see Figure~\ref{fig:ParadigmShift}.

While regular time leads to Hilbert spaces, and inherently to the Euclidean  norm, 
our approach is to revise the Euclidean view of geometry in this context
in favor of an approach based on alternative metrics which become necessary 
if we postulate certain analytical properties on the topological and metric structure of the space of
spike trains in combination with neuronal models and spiking neural networks as used in neuromorphic computing.
Our paper is a contribution to the mathematical foundation of SNNs 
by elaborating on postulates on the topology of the vector space of spike trains in terms of sequences of weighted Dirac impulses. 
Our approach is a follow-up of~\cite{MoserNatschlaeger2014Stability,Moser15a,Moser16EBCCSP,Moser2017Similarity,MoserLunglmayr2019QuasiIsometry} which discusses the discrepancy measure as a special range-based metric for spike trains for which a quasi-isometry relation for threshold-based sampling between analog  signals as input and  spike trains as output of a leaky-integrate-and-fire neuron can be established.
In this paper, we go beyond sampling and show that range-based metrics also play a special role  
for the mathematical conception of error, resp., deviation analysis based on spike trains, and thus for understanding, analyzing, and 
bounding error propagation of spiking neural networks (SNNs). 
 
The paper is structured as follows. In Section~\ref{s:Preliminaries}, 
we fix notation and prepare preliminaries from SNNs based
on leaky integrate-and-fire neuron model and the mathematical idealized assumption of instantaneous realization of events due to 
an impulse input. In this context, we discuss proposed variants of re-initialization after the firing event 
and revise the re-initialization mode known as {\it reset-by-subtraction} in the general setting of spike amplitudes 
that exceed multiples of the threshold. Such high spike amplitudes can arise in SNNs by scaling the input channels by weights.
This way we introduce the {\it reset-to-mod} re-initialization which results 
actually from a consequent application of the assumption of instantaneous events to {\it reset-by-subtraction}.
The resulting operation can be considered as modulo division. 
{\it reset-to-mod} allows to take a broader view on the mathematics of LIF as endomorphism that operates on the vector space of
spike trains of weighted Dirac impulses. 
This way we study the resulting LIF operator, specify the above motivated postulates in Section~\ref{s:Postulates} and derive a solution based on a topological argument in Subsection~\ref{ss:AlexiewiczTopology}, leading to the  Alexiewicz norm for integrate-and-fire (IF). 
We generalize this norm for leaky integrate-and-fire (LIF) and, as a main result, we show in Subsection~\ref{ss:Quantization} that the LIF operator
acts as spike quantization in the grid given by the corresponding norm. 
The related inequality turns out to be useful when it comes to the analysis of error and signal propagation due to perturbations or
added spikes in the input channels. Section~\ref{s:Additive} focuses on the 
effect of added spikes in the input channel, which results in a Lipschitz-style upper bound for the LIF model and feed-forward SNNs in general.
This analysis also shows that there can be a principle change in the input-output characteristic
along the transition from  integrate-and-fire with zero leakage to arbitrarily small leaky parameters. 
Section~\ref{s:Evaluation} reports on simulations that illustrate our theoretical approach.

\section{Preliminaries}
\label{s:Preliminaries}
First we recall the LIF neural model with its computational variants and fix notation. 
It has a long history which goes back to~\cite{lapicque_recherches_1907}. For an overview on its motivation and relevance in neuroscience and neuromorphic computing see~\cite{bookGerstner2014}, ~\cite{Daya_2001_book} and~\cite{snnTorch2021}.
Situated in the middle ground between biological plausibility and technical feasibility, the LIF model 
abstracts away the shape and profile of the output spike. This way spikes are mathematically represented as Dirac delta impulses.

Therefore, we start with spike trains as input signals which we assume to be given mathematically as sequences of weighted Dirac impulses, i.e.,
\begin{equation}
\label{eq:spiketrain}
\eta(t) := ([a_i; t_i]_i)(t):= \sum_{i \in \mathbb{N}_0} a_i \, \delta_{t_i}(t),
\end{equation}
where $a_i \in \mathbb{R}$ and $\delta_{t_i}$ refers to a Dirac impulse shifted by $t_i$.
$\mathbb{N}_0$ means that there is no bound for the number of spikes, though for convenience we assume that for each spike train the number is finite.
For convenience and without loss of generality, to ease notation below we assume that  $t_0=0$ and $a_0 = 0$. The empty spike train is denoted by $\emptyset$.

$(\mathbb{S}, +, \cdot)$ denotes the vector space of all spike trains~(\ref{eq:spiketrain}) based on usual addition and scaling, 
which later on will be equipped with a metric $d(.,.)$, resp. norm $\|.\|$,  to obtain the metric space $(\mathbb{S}, d)$, respectively, normed space 
$(\mathbb{S}, d)$. 

Mathematically, the LIF neuron model is actually an endomorphism, $\mbox{LIF}_{\vartheta, \alpha}: \mathbb{S} \rightarrow \mathbb{S}$, that is determined by two parameters, the threshold $\vartheta>0$ and the leaky time constant $\alpha>0$ and the mode for resetting the neuron after firing, respectively, the charging/discharging event. 
In this paper we consider three reset modes, {\it reset-to-zero}, {\it reset-by-subtraction} and {\it reset-to-mod}.
According to~\cite{snnTorch2021}, {\it reset-to-zero} means that the potential is reinitialized to zero after firing, while 
{\it reset-by-subtraction} subtracts the $\vartheta$-potential $u_{\vartheta}$ from the membrane's potential that triggers the firing event.
As a third variant we use the term {\it reset-to-mod}, which can be understood as 
instantaneously cascaded application of {\it reset-by-subtraction} according to the factor by which the membrane's potential exceeds the threshold which results in a modulo computation. This means, in the {\it reset-to-mod} case the 
re-initialization starts with the residue after subtracting the integral multiple of the threshold from the membrane's potential at firing time.
 
Setting $t_0^{(1)}:=0$ (where the upper index indicates the layer, here the output of LIF) and $\eta_{\tiny in}(t) := ([a_i; t_i^{(0)}]_{i})(t)$
the mapping 
$$
\sum_{i \in \mathbb{N}_0} b_i \, \delta_{t_i^{(1)}} = \eta_{\tiny out} = \mbox{LIF}_{\vartheta, \alpha}(\eta_{\tiny in})
$$
is recursively given by 
\begin{equation}
\label{eq:LIF}
t_{i+1}^{(1)} =
\inf\left\{t\geq t_i^{(1)}
: \,
\left| u_{\vartheta, \alpha}(t_i^{(1)}, t)  \right| 
\geq \vartheta\right\}, 
\end{equation}
where 
\begin{equation}
\label{eq:u}
u_{\alpha}(t_i, t) := 
 \int^t_{t_i}
	e^{-\alpha (\tau - t_i)} 
	\left(		\eta_{\tiny in}(\tau) - \mbox{discharge}(t_i, \tau) \right) d\tau
\end{equation}
models the dynamic change of the neuron membrane's potential after an input spike event at time $t_i$  
(based on the assumption of instantaneous increase, resp. decrease).
At the moment when the absolute value of the membrane potential touches the threshold-level, $\vartheta>0$, an output spike is generated 
whose amplitude is given by 
$b_{i+1} = +\vartheta$ or $= - \vartheta$ depending on whether the membrane's potential $u_{\alpha}$ 
in~(\ref{eq:LIF}) is positive or negative.

The process of triggering an output spike is actually a charge-discharge event that is followed by the
re-initialization of the membrane's potential modeled by an instantaneously acting discharge process
\begin{equation}
\label{eq:reset}
\mbox{discharge}(t_i^{(1)}, \tau) := 
\left\{
\begin{array}[2]{lcl}
	u_i\,\delta_{t_i^{(1)}}(\tau) & \ldots & \mbox{for {\it reset-to-zero}}, \\
	\mbox{sgn}(u_i)\, \vartheta\, \delta_{t_i^{(1)}}(\tau)  & \ldots & \mbox{for {\it reset-by-subtraction}},  \\
        \mbox{sgn}(u_i) [u_i/\vartheta] \, \vartheta \, \delta_{t_i^{(1)}}(\tau)  & \ldots & \mbox{for {\it reset-to-mod}},	
\end{array}
\right.
\end{equation}
where $u_i:= u_{\alpha}(t_{i-1}, t_i)$, $\mbox{sgn}(x) \in \{-1,0,1\}$ is the signum function and 
\begin{equation}
\label{eq:[]}
[x] := \mbox{sgn}(x)\max\{k \in \mathbb{Z}: k \leq |x|\} 
\end{equation}
realizes integer quantization by truncation.

The integration in (\ref{eq:LIF}) models the voltage in an RC circuit as response to current impulses.
Note that the immediate reset without delay in (\ref{eq:LIF}) is an idealization from biology or hardware realizations.
Anyway, since in practical realizations spikes are sparse in time (or should be) this idealization is a justifiable approximation.
{\it reset-to-zero} and {\it reset-by-subtraction} can show quite different behavior if the spike amplitudes are large.

The {\it reset-by-subtraction} mode can be understood as compensation event so that the net voltage balance of the 
spiking event equals zero, i.e. in case of an output spike with amplitude $\vartheta$ the membrane is actually
discharged by this amount. Accordingly, though not always made clear in the literature, see for example~\cite{snnTorch2021},
this assumption has the subtle consequence that an increase of the membrane potential $u$ by multiples $[u/\vartheta]$ of the threshold level 
$\vartheta$ results in a discharge of the membrane's potential by the same amount, 
i.e., $[u/\vartheta]\vartheta$.

This can also be seen virtually as sequence of $[u/\vartheta]$-many $\vartheta$-discharge actions, which  acting in sequence
in instantaneous time produce the same result, that is an output spike with amplitude 
$[u/\vartheta]\vartheta$.
Here we express the amplitude of the output spikes as multiple of the unit in terms of the threshold potential $\vartheta$.
For example, consider a single spike $\eta_{\tiny in}(t):= a_1 \delta_{t_1}(t)$ with large amplitude $|a_1|> \vartheta$. 
Note that due to the idealization of instantaneous actions of charge and discharge events 
the discharge model in the {\it reset-by-subtraction} mode implies that  
$\eta_{\tiny in} = a_1 \delta_{t_1}$ is mapped to 
\begin{equation}
\label{eq:reset[]}
b_1 \delta_{t_1} = \mbox{LIF}_{\vartheta, \alpha}(a_1 \delta_{t_1})  , \,\,\,  b_1 = [a_1/\vartheta]\, \vartheta.
\end{equation}
Fig.~\ref{fig:LIF} illustrates LIF with {\it reset-by-subtraction}.  
\begin{figure}[ht]
	\centering
  \includegraphics[width=1\textwidth]{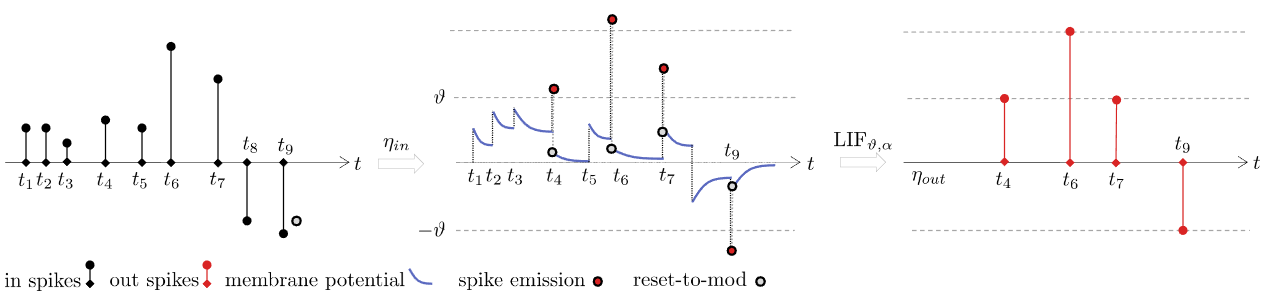}
		\caption{LIF in continuous time with {\it reset-by-subtraction}, resp. {\it reset-to-mod}; at $t_6$ the amplitude $a_6 \in (2\vartheta,3\vartheta) $ of the input spike, which causes a two times cascaded  {\it reset-by-subtraction} resulting in an output spike amplitude $b_6 = 2 \vartheta = [a_6/\vartheta]\,\vartheta$.}
	\label{fig:LIF}
\end{figure}

Depending on the research and application context,  discrete approximations 
of the LIF model~(\ref{eq:LIF}) become popular, 
particularly to simplify  computation and to make the application of deep learning methods to spike trains easier~\cite{snnTorch2021}.
Under the assumptions 
\begin{itemize}
\item[(i)] continuous time $t \in [0,\infty)$ is replaced by discrete time $n \Delta t \in \Delta t\, \mathbb{N}_0$, 
where $\Delta t \ll \alpha$;
\item[(ii)] instantaneous increase, respectively decrease of the membrane's potential~(\ref{eq:u}); 
\end{itemize}
we obtain a discrete computational model, where the
input signal $\eta_{in} = \sum_i a_i \delta_{t_i}$ in continuous time is replaced by the sequence 
\begin{equation}
\label{eq:ahat}
\hat{a}_k := 
\left\{
\begin{array}[pos]{lcl}
a_i & & k = [t_i/\Delta t], \\
0 & & \mbox{else}\\	
\end{array}
\right.
\end{equation}
which is well-defined if $\Delta t$ is chosen sufficiently small so that at most one Dirac impulse hits a time interval
$I_k = [k \Delta t, (k+1) \Delta t)$.
The amplitudes $\hat{b}_n$, $n = 0,1, \ldots$, of the output spike train are defined as for continuous time.
This way, finally we get the discrete LIF model, 
$\widehat{\mbox{LIF}}_{\vartheta, \beta, \Delta t}: \mathbb{R}^{\mathbb{N}_0} \rightarrow \mathbb{R}^{\mathbb{N}_0}$,
$(\hat{b}_k)_k = \widehat{\mbox{LIF}}_{\vartheta, \beta, \Delta t}((\hat{a}_k)_k)$,
as outlined in Algorithm~\ref{alg:DiscreteLIF}.
\begin{algorithm}
\caption{Simplified Discrete LIF model $\widehat{\mbox{LIF}}_{\vartheta, \beta, \Delta t}$}
\label{alg:DiscreteLIF}
\noindent
{\bf Step 0}:  Initialization: $\hat a = (\hat{a}_k)_k$, $u_0:=0$, $\hat{b}_0 = 0$,
$\beta := (1 - \frac{\Delta t}{\alpha})$;
\\
\noindent
{\bf Step 1}: Update Membrane Potential: 
$u_{n+1}:= \beta\, u_{n} +  \hat{a}_n - \hat{b}_n$ 
\noindent\\
{\bf Step 2}: Check Threshold: Update time $n\mapsto n+1$ and check whether 
$|u_{n}| \geq \vartheta$.
If 'no', then set $\hat{b}_n:= 0$ and repeat {\bf Step 1}; if 'yes' then output 
a spike at time step with amplitude $u_{n}$ and move on to {\bf Step 3}.
\noindent\\
{\bf Step 3}:  Discharge Event: According to the re-initialization mode set
\begin{equation}
\label{eq:DisReset}
\hat{b}_n := 
\left\{
\begin{array}[2]{lcl}
	u_n  									& \ldots & \mbox{for {\it reset-to-zero}}, \\
	\mbox{sgn}(u_n) \vartheta			& \ldots & \mbox{for {\it reset-by-subtraction}}, \\
	{\left[u_n / \vartheta \right]}\vartheta   & \ldots & \mbox{for {\it reset-to-mod}}.	
\end{array}
\right.
\end{equation}
\\
{\bf Step 4:}  Repeat steps {\bf Steps 1,2,3} until all input spikes are processed.
\end{algorithm}

With this mathematical clarification of the LIF model, in continuous and discrete time, we are in the position
to study integrate-and-fire as spike quantization and provide upper bounds for the quantization error in Section~\ref{ss:Quantization}. 

In this paper we consider feed-forward spiking neural networks, $\mbox{SNN}: \mathbb{S}^{N_0} \rightarrow \mathbb{S}^{N_L}$ which are mappings given by weighted directed acyclic graphs $(V,E)$ connecting LIF units with fixed parameters 
$\vartheta$ and $\alpha$.  
$\mbox{SNN}$ takes $N_0$ spike trains as input and maps them to $N_L$ output spike trains.
The underlying graph can be arranged in hierarchies starting from the first layer $1$ up to layer $L$.
We enumerate the LIF nodes in the $k$th layer by $(k, i_k)$, where $i_k \in \{1, \ldots, N_k\}$.

\begin{wrapfigure}{r}{9cm}
  \includegraphics[width=8.5cm]{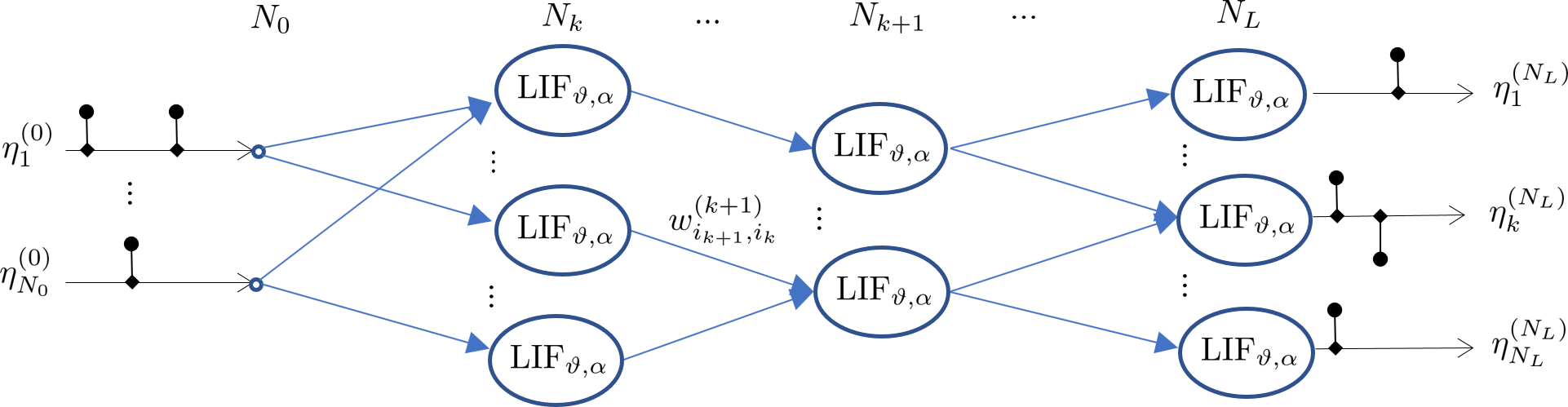}
  \caption{SNN as weighted directed acyclic graph.}
		\label{fig:SNN}
\end{wrapfigure} 

For convenience we consider the input spike trains as $0$ layer.
The weight $w^{(k+1)}_{i_{k+1}, i_{k}}$ of an edge $[(k, i_k), (k+1, i_{k+1})] \in E$ 
connecting the $i_k$-th neuron in the $k$-th layer with the $i_{k+1}$-th neuron in the $(k+1)$-th layer 
rescales accordingly the weights of the spike train being transmitted from the former neuron to the latter. 
See Fig.~\ref{fig:SNN} for an illustration.
This way the mapping $\mbox{SNN}$ can be represented by the tuple of weight matrices $W = [W^{(1)},\ldots,W^{(N_L)}]$,
 where 
\begin{eqnarray}
\label{eq:SNNW}
\mbox{SNN} & = & [W^{(1)},\ldots,W^{(N_L)}], \\
W^{(k+1)} &=& (w^{(k+1)}_{i_{k+1}, i_{k}}) \in \mathbb{R}^{N_{k+1} \times N_k}\nonumber.
\end{eqnarray}

\section{Which Topology for the Space of Spike Trains is Appropriate?}
\label{s:Postulates}
Our approach starts with two main postulates a topology for spike trains should satisfy, see Fig.~\ref{fig:Postulates}:
\begin{center}
\fbox{\parbox{13.5cm}{
{\it Postulate 1:}~Two spike trains that differ only by small delays of their spikes or small additive noise should be considered close where the notion of closeness should not depend on the number of spikes. 
\\
{\it Postulate 2:}~Small perturbations in the system's configuration parameters should end up in similar input-output behavior, e.g., if the threshold deviates only by some small value. 
}
}
\end{center}
\begin{figure}[ht]
	\centering
  \includegraphics[width=0.9\textwidth]{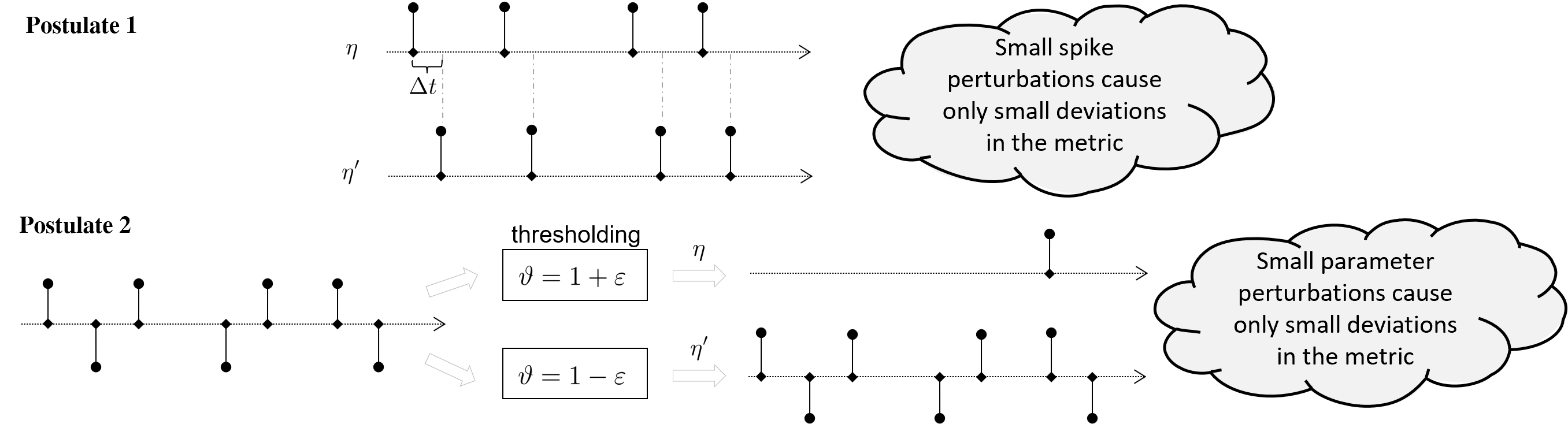}
		\caption{Postulates for an adequate metric $d(.,.)$ for spike strains.}
	\label{fig:Postulates}
\end{figure}
Note that the widely spread Euclidean approach based on summing up squared differences does not meet these postulates.
For example, in the context of back propagation expressions of the type $d_E(\eta, \eta')^2:= \sum_i (t_i-t_i')^2$ are used, 
see, e.g.~\cite{Bohte2000}. Apart from the problem that this definition is only well defined if there is a one-to-one correspondence between the spikes in the first and the second spike train. This ansatz may be useful for certain algorithms in a certain setting, but due to the lack of well-definedness it is not suitable for an axiomatic foundation of a generally valid theory. For example, this ansatz is not well-defined in the scenario of Postulate 2 if the first spike train is empty and the second not. Also Postulate 1 is problematic as the error depends on the number of spikes. 
A large error can  result from a large delay of a single spike or of many small delays. 
See also~\cite{MoserNatschlaeger2014Stability,Moser15a}. 
In addition, the sign of the spike is not taken into account.

\subsection{Alexiewicz Topology}
\label{ss:AlexiewiczTopology}
In order to get an idea about signals that should be considered close in the topology 
let us consider the set $C$ of all sub-threshold input spike trains to a LIF neuron 
$\mbox{LIF}_{\vartheta, \alpha}: \mathbb{S} \rightarrow \mathbb{S}$.
Note that $C$ is the pre-image of LIF of the empty spike trains, i.e., 
$C := \mbox{LIF}_{\vartheta, \alpha}^{(-1)}(\{\emptyset)\}$. 
$C$  is obviously not a closed set, as, e.g., all spike trains $\eta_k := (\vartheta-1/k) \delta_{t_1}$ are below threshold but not its
pointwise limit. Taking also all limits into account we obtain a notion of closure $\overline{C}$ of $C$.
$\overline{C}$ can be characterized in the following way, see~\ref{ss:th:closureIF}.

\begin{lemma}
\label{th:closureIF}
For a leaky integrate-and-fire neuron $\mbox{LIF}_{\vartheta, \alpha}: \mathbb{S} \rightarrow \mathbb{S}$ with $0 \leq \alpha < \infty $ we have:
\begin{equation}
\label{eq:closure}
\eta = \sum_i a_i \delta_{t_i}  \in \overline{\mbox{LIF}_{\vartheta, \alpha}^{(-1)}(\{\emptyset)\}} \Longleftrightarrow
\max_{n}\left|\sum_{i=1}^n a_i e^{-\alpha (t_{n} - t_i)}\right| \leq \vartheta. \nonumber
\end{equation}
\end{lemma}

\begin{wrapfigure}{r}{5cm}
  \includegraphics[width=4.5cm]{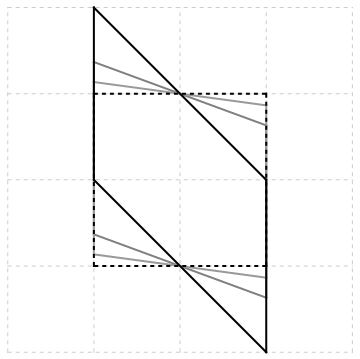}
	\caption{$\|.\|_{A,\alpha}$-unit balls for spike trains $\eta = a_1 \delta_{0} + a_2 \delta_{1}$ for
	$\|.\|_{A,\alpha}$ with $\alpha = 0$ (solid black), $\alpha = \infty$ (dashed), $\alpha = 1$ and $\alpha = 2$ (gray)}
  		\label{fig:Unitballs}
\end{wrapfigure}

Note that 
\begin{equation}
\label{eq:A}
\|([a_i; t_i])_i\|_{A, \alpha}:=
\max_n\left| 
\sum_{j=1}^n a_j e^{-\alpha (t_n - t_j)}
\right| 
\end{equation}
defines a norm on the vector space $\mathbb{S}$, which justifies the $\|.\|$ notation. 
As immediate consequences from the definition (\ref{eq:A}) we obtain
\begin{equation}
\label{eq:Aalt}
\|\eta\|_{A, \alpha} = \inf\left\{\vartheta >0:  \mbox{LIF}_{\vartheta, \alpha}(\eta)=\emptyset \right\},
\end{equation}
and
\begin{equation}
\label{eq:Aalt}
\forall \alpha, \beta \in (0,\infty), \eta \in \mathbb{S}:\| \mbox{LIF}_{\vartheta, \alpha}(\eta(\cdot))\|_{A, \alpha} = 
\| \mbox{LIF}_{\vartheta, \beta}\left(\eta(\alpha/\beta\, \cdot)\right)\|_{A, \beta}. 
\end{equation}

This way, $\overline{\mbox{LIF}_{\vartheta, \alpha}^{(-1)}(\{\emptyset)\}}$ turns out to be the ball 
$B_{A,\alpha}(\vartheta)$ centered at $\emptyset$ of radius $\vartheta$ w.r.t the norm $\|.\|_{A, \alpha}$.
For $\alpha =0$ the length of the  time intervals between the events do not have any effect, and we get
the norm  $\|(a_i)_i\|_{A, 0} = \max_n \left|\sum_{i=1}^n a_i\right|$.
By looking at $a_i$ as width of a step of a walk up and down along a line, $\|.\|_{A, 0}$ marks 
the maximum absolute route amplitude of the walk.

Range measures are studied in the field of random walks in terms of an asymptotic distribution resulting from diffusion process~\cite{Finch2018, Jain1968}. 
A similar concept is given in terms of the diameter $\|(a_i)_i\|_D$ of a walk, i.e.,
$\|(a_i)_i\|_D := \max_{1\leq m\leq n\leq N} \left|\sum_{i=m}^n  a_i\right|$,
which immediately can be generalized to 
$\|(a_i)_i\|_{D, \alpha} := \max_{1\leq m\leq n\leq N} 
| \sum_{j=m}^n a_j e^{-\alpha (t_n - t_j)} |
$.
Note that 
$\|(a_i)_i\|_{A, \alpha} \leq \|(a_i)_i\|_{D, \alpha} \leq 2 \|(a_i)_i\|_{A, \alpha}$
 stating the norm-equivalence of $\|.\|_{A, \alpha}$ and $\|.\|_{D, \alpha}$.

While the unit ball $B_{A, 0}$ w.r.t $\|.\|_{A,0}$ can be understood by shearing the hypercube $[-1,1]^N$, see~\ref{A:normA}, the 
geometric characterization of the related unit ball $B_{D,0}$ of $\|.\|_{D,0}$ is more tricky, see~\cite{Moser12UnitBall}.

For an illustration of the corresponding unit balls for two spikes (2D case) see 
Fig.~\ref{fig:Unitballs}. 

These concepts are related to the more general concept of {\it discrepancy} measure, see~\cite{Chazelle2000,Moser11TPAMI}, which goes back to Hermann Weyl~\cite{Weyl1916} and is defined on the basis of a family $\mathcal{F}$ of subsets $F$ of the universe of discourse, i.e., 
$\mu((a_i)_i) = \sup_{F \in \mathcal{F}}| \sum_{i \in F} a_i|$. For $\|.\|_A$ the family $\mathcal{F}$ consists of all index intervals $\{0, \ldots, m\}$, while for $\|.\|_D$ the family 
$\mathcal{F}$ consists of all partial intervals $\{m, \ldots, n\}$, $m, n \in \{1, \ldots, N\}$. 
Therefore we refer particularly to $\|.\|_A$, resp. $\|.\|_D$, as example of a {\it discrepancy measure}. 

An analogous concept, $\|f\|:= \sup_{[a,b]} |\int_{[a,b]} f d\mu|$,  can be defined for functions $f$ and 
tempered distributions such as Dirac delta impulses by using integrals instead of the discrete sum, which is known in the literature as {\it Alexiewicz} semi-norm~\cite{Alexiewicz1948}. As spike trains live in continuous time, in the end our topology we are looking for is the Alexiewicz topology, which 
meets the postulates above. However, most of the reasoning and proofs in the context of this paper can be boiled down 
to  discrete sequences, hence utilizing the discrepancy norm.

\subsection{Spike Train Quantization}
\label{ss:Quantization}
Interestingly, as pointed out by~\cite{MoserLunglmayrESANN2023}, LIF  can be understood as a $\|.\|_{A,\alpha}$-quantization operator satisfying
\begin{equation}
\label{eq:quantization}
\|\mbox{LIF}_{\vartheta, \alpha}(\eta) - \eta\|_{A, \alpha} < \vartheta.
\end{equation}
\cite{MoserLunglmayrESANN2023} provides a proof of (\ref{eq:quantization})
for weighted Dirac impulses as input signal to the LIF, see~\ref{A:Dirac:th:quantization}.
Here, first we note that (\ref{th:quantizationDirac}) also applies to the discrete version of Algorithm~\ref{alg:DiscreteLIF}. The proof is analogous. Second, we state a generalization to piecewise continuous functions.  
This way, we show that LIF acts like a
signal-to-spike-train quantization.
The generalization of Dirac pulses to more general classes of signals 
is especially important for a unifying theory that combines analog spike sampling and SNN-based spike-based signal processing.
An extension to the general class of locally integrable functions is also possible but requires the introduction of the Henstock-Kurzweil integral~\cite{Kurtz2004TheoriesOI}, which is postponed to future research. 
\begin{theorem}
\label{th:quantization}
(\ref{eq:quantization})  also holds for piecewise continuous functions $\eta$, i.e., functions having at most a finite number of 
discontinuities. 
\end{theorem}

\begin{proof}
The idea is to construct a $\hat{\eta} = \sum_i a_i \delta_{t_i} \in \mathbb{S}$ such that 
$\int_{t_0}^{t_i}	\hat{\eta}(t) e^{\alpha t} dt = \int_{t_0}^{t_i} \eta(t) e^{\alpha t} dt$ for all $t_i$.
This can be achieved by utilizing the {\it mean value theorem for integrals}. 
First, partition the time domain into intervals $U_k = (u_{k-1}, u_k)$ on which $\eta$ is continuous. 
On $U_k$, the mean value theorem guarantees the existence of $s_k \in U_k$ such that 
$\int_{U_k} \eta(t) e^{\alpha t} dt = |U_k| \eta(s_k)e^{\alpha s_k}$. 
Then define the sequence $(t_i)_i$ consisting of all $s_k$ and the $u_k$. For $t_i = s_k$ define 
$a_i := |U_k| \eta(s_k)$ and for  $t_i = u_k$ define 
$a_i := \lim_{\varepsilon \rightarrow 0} \int_{u_k-\varepsilon}^{u_k+\varepsilon} \eta dt$.
Refine this partition so that also the time points $t_i^*$ of the output spikes of 
$\mbox{LIF}_{\vartheta, \alpha}(\eta) = \sum_i b_i \delta_{t_i^*}$
are taking into account as border points of the $U_i$ intervals. 
This way we obtain for all $t_i$:
\begin{equation}
\label{eq:generalQuant1}
\int_{t_0}^{t_i} \eta(t) e^{- \alpha (t_i - t)} dt = \int_{t_0}^{t_i} \hat{\eta}(t) e^{- \alpha (t_i - t)} dt.
\end{equation}
Moreover, since all time points $t_i^*$ are contained in $(t_i)_i$, we also have 
\begin{equation}
\label{eq:generalQuant2}
\mbox{LIF}_{\vartheta, \alpha}(\eta) = \mbox{LIF}_{\vartheta, \alpha}(\hat{\eta}).
\end{equation}
Putting (\ref{eq:generalQuant1}) and (\ref{eq:generalQuant2}) together closes the proof.
\end{proof}

Fig.~\ref{fig:quantization} illustrates the quantization for different values of $\alpha$ w.r.t $\|.\|_{A, \alpha}$.
Note that for $\alpha \rightarrow \infty$ we obtain the standard $\|.\|_{\infty}$-quantization.
\begin{figure}[ht]
	\centering
  \includegraphics[width=1\textwidth]{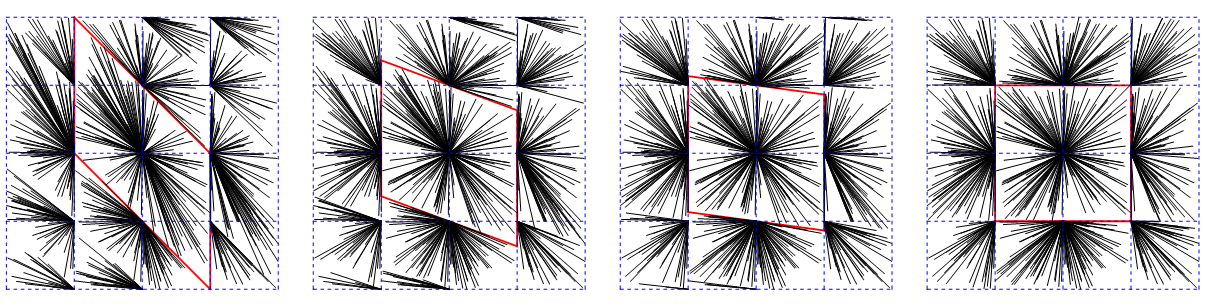}
		\caption{Quantization w.r.t $\|.\|_{A,\alpha}$, $\alpha \in \{0,1,2,\infty\}$ for spike trains $\eta = a_1\delta_0 + a_2\delta_1$ 
		with random $a_i \in [-2,2]$; the corresponding unit balls are marked red;
		the arrows are connecting points with their quantization points.}
	\label{fig:quantization}
\end{figure}

Like for threshold-based sampling~\cite{Moser2017Similarity,MoserLunglmayr2019QuasiIsometry} we also obtain 
quasi isometry in the discrete case, though the situation and the way of proving it is different.
Here in this context, we get it as a byproduct of~Theorem~\ref{th:quantization}.
\begin{corollary}[LIF Quasi Isometry]
\label{cor:QuasiIsometry}
The norm $\|.\|_{A, \alpha}$, $\vartheta>0$, establishes quasi isometry for the LIF neuron model, i.e.,
\begin{equation} 
\label{eq:quIsometry}
\|\eta_1 - \eta_2  \|_{A, \alpha}  - 2 \vartheta \leq
\|\mbox{LIF}_{\vartheta, \alpha}(\eta_1) - \mbox{LIF}_{\vartheta, \alpha}(\eta_2)  \|_{A, \alpha} \leq 
\|\eta_1 - \eta_2  \|_{A, \alpha}  + 2 \vartheta
\end{equation}
and asymptotic isometry, i.e., 
\begin{equation}
\label{eq:asymIsometry}
\lim_{\vartheta\rightarrow 0} \|\mbox{LIF}_{\vartheta, \alpha}(\eta_1) - \mbox{LIF}_{\vartheta, \alpha}(\eta_2)  \|_{A, \alpha}
= \|\eta_1 - \eta_2  \|_{A, \alpha} 
\end{equation}
for all $\eta_1, \eta_2 \in \mathbb{S}$.
\end{corollary}

Theorem~\ref{th:quantization} together with the quasi isometry property (\ref{eq:quIsometry}) immediately gives an answer to our Postulates 1 and 2 in Section~\ref{s:Postulates} in terms of Corollary~\ref{cor:BoundonLag} and Corollary~\ref{cor:BoundonThetaPerturbation}. 
Because of the discontinuity of thresholding the best what we can expect is an error bound in order of the threshold $\vartheta$.
The error bound~(\ref{eq:lag1}) in Corollary~\ref{cor:BoundonLag} caused by a small lag is remarkable as it 
asymptotically depends  only on the threshold and the maximal spike amplitude and not, e.g., 
on the spike frequency. This property is typical for the Alexiewicz norm and related metrics such as the discrepancy measure and contrasts the Euclidean geometry and its related concept of measuring correlation, see also~\cite{MoserSB11SIMBAD}.
For the proof we refer to~\ref{A:Lag}.
\begin{corollary}[Error Bound on Lag]
\label{cor:BoundonLag}
For $\eta = \sum_i a_i \delta_{t_i}\in \mathbb{S}$ 
and sufficiently small $\Delta t$ we get the error bound: \\
\begin{eqnarray}
\label{eq:lag1}
 \left\| 
\mbox{LIF}_{\vartheta, \alpha}(\eta(\cdot - \Delta t)) -
\mbox{LIF}_{\vartheta, \alpha}(\eta(\cdot))   
\right\|_{A, \alpha} & \leq & \nonumber \\
 \max_i|a_i|+ 2\,\vartheta + \Delta t \,\alpha (\|\eta\|_{A, \alpha} + \max_i|a_i|)+ O(\Delta t^2). & & \nonumber
\end{eqnarray}
\end{corollary} 
(\ref{th:quantization}) together with the triangle inequality of the norm gives
\begin{eqnarray}
\left\| 
\mbox{LIF}_{\vartheta + \varepsilon, \alpha}(\eta)  - \eta + \eta - \mbox{LIF}_{\vartheta, \alpha}(\eta)   
\right\|_{A, \alpha} 
& \leq &  \nonumber \\
\left\| 
\mbox{LIF}_{\vartheta + \varepsilon, \alpha}(\eta)  - \eta 
\right\|_{A, \alpha} 
+ 
\left\| 
\mbox{LIF}_{\vartheta, \alpha}(\eta) - \eta
\right\|_{A, \alpha} & \leq & 2 \vartheta + \varepsilon, \nonumber 
\end{eqnarray} 
proving Corollary~\ref{cor:BoundonThetaPerturbation}.

\begin{corollary}[Error Bound on Threshold Perturbation]
\label{cor:BoundonThetaPerturbation}
For $\varepsilon>0$ we have
\begin{equation}
\label{eq:BoundonThetaPerturbation}
\sup_{\eta \in \mathbb{S}}\left\| 
\mbox{LIF}_{\vartheta + \varepsilon, \alpha}(\eta)   
-
\mbox{LIF}_{\vartheta, \alpha}(\eta)   
\right\|_{A, \alpha} \leq 2 \vartheta + \varepsilon.
\end{equation}
\end{corollary} 
(\ref{th:quantization}) can also  be interpreted as spike train decomposition into a part that consists of spikes with amplitudes that are 
signed multiples of the threshold and a sub-threshold residuum. It is interesting that the first part can be further decomposed into 
a sum of unit Alexiewicz norm spike trains. See~\ref{A:Decomposition} for an example.
\begin{theorem}[Spike Train Decomposition]
\label{cor:ADecomposition}
For any $\eta \in \mathbb{S}$ there is a $\psi \in \mathbb{S}$ with  spike amplitudes that are integer multiples of the 
threshold and 
a below-threshold residuum spike train $\rho \in \mathbb{S}$ with $\|\rho\|_{A, \alpha}< \vartheta$, 
such that $\eta = \psi + \rho$, where $\psi = \mbox{LIF}_{\vartheta, \alpha}(\eta)$.
Moreover, $\psi$ can be represented as sum of $\|.\|_{A,0}$-unit spike trains $\Delta\eta_r$, 
$r \in \{1, \ldots, a\}$, $a:=\|\psi\|_{A,0}$, i.e.,
\begin{equation}
\psi= \sum_{r=1}^{a} \Delta\eta_r, 
\end{equation}
where $\|\Delta\eta_r\|_{A,0}=1$ for all $r$.
\end{theorem}
\begin{proof}

Let $\eta_0 := \psi$ be the initial spike train with integer spike amplitudes $a_i^{(0)} \in \mathbb{Z}$.
Assume that  $\|\eta_0\|_{A,0}>1$. 
For convenience we define a sum over an empty index set to be zero, i.e.,  $\sum_{i \in \emptyset} a_i = 0$.
We will recursively define a sequence 
\begin{equation}
\label{eq:etar}
\eta_r = \sum_{i=1}^{N^{(r)}} a_i^{(r)}  \delta_{t_i},
\end{equation}
of spike trains for $r= 1, \ldots, N^{(0)}$ such that $\|\eta_{r}- \eta_{r-1}\|_{A,0}=1$ and $\eta_{N^{(0)}}=\emptyset$.
Denote $N^{(r)}:=\|\eta_r \|_{A, 0}$.

\begin{figure}
\centering
\includegraphics[width=8cm]{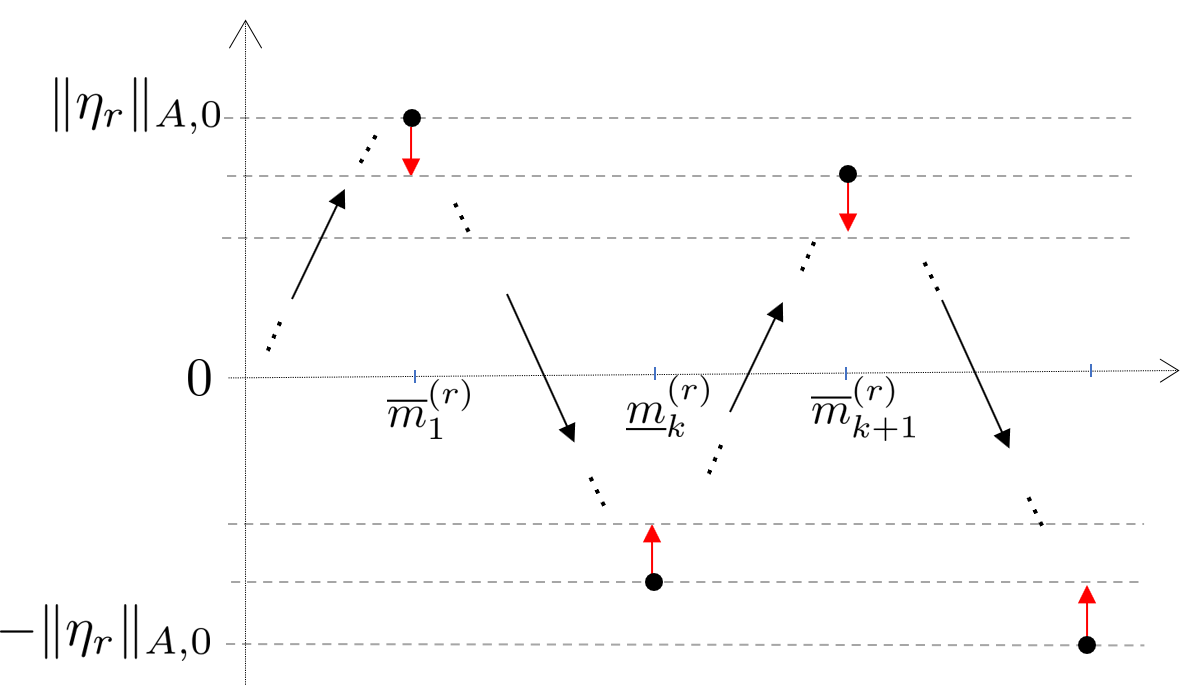}
\caption{Peaks in spike decomposition algorithm. 
	The subtraction of $\Delta\eta_r$ results in shifting the peaks towards the zero line, indicated by the red arrows. }
		\label{fig:walkA}
\end{figure}

If $N^{(r)}\geq 2$, then according to Fig.~(\ref{fig:walkA}) we consider the peaks in the walk $S_k = \sum_{i=1}^k a_i^{(r)}$.
Without loss of generality let us assume that the first peak is positive. 
For this we define recursively the corresponding top and bottom peak indexes $\overline{m}_k^{(r)}$, resp. $\underline{m}_k^{(r)}$ as follows.
\begin{eqnarray}
\label{eq:topPeak}
\overline{m}_1^{(r)} & :=&  \min\{k>0: \sum_{i=1}^k a_i^{(r)} = N^{(r)}\}, \nonumber \\
\overline{m}_{k+1}^{(r)}  & := & \min\{k> \overline{m}_{k}^{(r)}: \sum_{i=1}^k a_i^{(r)} \geq N^{(r)}-1\},
\end{eqnarray}
and, analogously, 
\begin{equation}
\label{eq:bottomPeak}
\underline{m}_{k}^{(r)}  := \min\{
m \in J=\{\overline{m}_k^{(r)}+1, \ldots, \overline{m}_{k+1}^{(r)} \}: 
\sum_{i=1}^m a_i^{(r)} = \min_{j \in J} \sum_{i=1}^j a_i^{(r)} \leq -1\}. 
\end{equation}

Based on (\ref{eq:topPeak}) and (\ref{eq:bottomPeak}) we define the spike train 
$\Delta \eta_{r+1} := \sum_i d^{(r+1)}_i \delta_{t_{i}}$ as follows.
Due to our assumption that the first peak is positive we define (otherwise $-1$)
\begin{equation}
\label{eq:a1}
d^{(r+1)}_{m_1}  :=  1. 
\end{equation}
For the subsequent peaks we consider the down and up intervals 
\begin{equation}
\label{eq:Jk}
\underline{J}_k  := \left\{\overline{m}_k^{(r)}+1, \ldots, \underline{m}_{k}^{(r)} \right\}, \, \, 
\overline{J}_k  :=  \left\{\underline{m}_k^{(r)}+1, \ldots, \overline{m}_{k+1}^{(r)} \right\}.
\end{equation}

We set $d^{(r+1)}_{i} := 0$ for all $t_i$ except the following cases. 
There are  two cases for down intervals (analogously for up intervals):
\begin{itemize}
\item Case A. $\sum_{i \in \underline{J}_k} a_i^{(r)} \leq -2$ and there is an index $i \in \underline{J}_k: a_i^{(r)}\leq -2$, 
then we set
\begin{equation}
\label{eq:d1}
d^{(r+1)}_{i}  :=  -2. 
\end{equation}
\item Case B. $\sum_{i \in \underline{J}_k} a_i^{(r)} \leq -2$ and there is no index $i \in \underline{J}_k: a_i^{(r)}\leq -2$, then there are at least two indexes $i_1, i_2$ such that $a_{i_1}^{(r)} + a_{i_2}^{(r)} \leq -2$. Thus, we set
\begin{equation}
\label{eq:d2}
d^{(r+1)}_{i_1}  := -1, \nonumber \\
 d^{(r+1)}_{i_2}  := -1. 
\end{equation}
\end{itemize}
Analogously, we define the spikes for the up intervals, i.e., again distinguishing two cases.
\begin{itemize}
\item Case A. $\sum_{i \in \overline{J}_k} a_i^{(r)} \geq 2$ and there is an index $i \in \underline{J}_k: a_i^{(r)} \geq 2$, 
then we set
\begin{equation}
\label{eq:d11}
d^{(r+1)}_{i} :=  2. 
\end{equation}
\item Case B. $\sum_{i \in \underline{J}_k} a_i^{(r)} \geq 2$ and there is no index $i \in \underline{J}_k: a_i^{(r)}\geq 2$, then there are at least two indexes $i_1, i_2$ such that $a_{i_1}^{(r)} + a_{i_2}^{(r)} \geq 2$. Thus, we set
\begin{equation}
\label{eq:d21}
d^{(r+1)}_{i_1} :=  1, \nonumber \\
 d^{(r+1)}_{i_2}  :=  1. 
\end{equation}
\end{itemize}
Not that $\|\Delta \eta_{r+1}\|_{A, 0}=1$ and $\|\eta_r - \Delta\eta_{r+1}\|_{A,0} = \|\eta_r\|_{A,0}-1$, 
since all peaks are shifted by $1$ towards the zero line.
Since in each step the $\|.\|_{A,0}$ reduced by $1$, $\|\psi\|_{A,0}$ many steps are sufficient to represent 
$\psi = \sum_r \Delta \eta_r$.
\end{proof}

\section{Additive Spike Errors and a Resonance Phenomenon}
\label{s:Additive}

\begin{wrapfigure}{t}{6.5cm}
  \includegraphics[width=6cm]{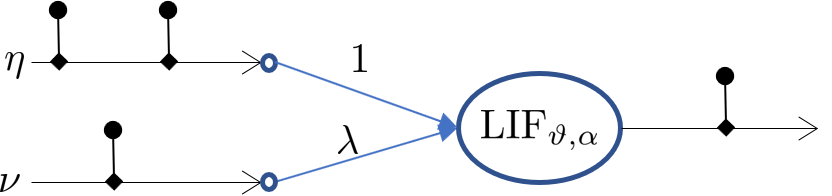}
  \caption{Scheme for additive signal error propagation through a single LIF model.}
		\label{fig:SchemeErrorPropLIF}
\end{wrapfigure}

In this section we first study the effect on the output of a single LIF neuron model 
when perturbing an input spike train $\eta$ by 
adding weighted spikes $\nu$, as illustrated in Fig.~\ref{fig:SchemeErrorPropLIF}.  

First of all we consider the special cases of zero and infinite leakage, i.e., $\alpha=0$, resp., $\alpha = \infty$, 
to obtain Lemma~\ref{lem:IFnu}. For the proof see~\ref{A:lem:IFnu}.
\begin{lemma}[Additive Error Bound for Integrate-and-Fire]
\label{lem:IFnu}
Let  $\mbox{LIF}_{\vartheta, \alpha}$ be a LIF neuron model with $\alpha \in \{0, \infty\}$ and {\it reset-to-mod} re-initialization, then:
\begin{equation}
\label{eq:IFnu} 
\forall \vartheta>0, \eta, \nu \in \mathbb{S}: \|\nu\|_{A,\alpha}\leq \vartheta \Rightarrow  
\|\mbox{LIF}_{\vartheta,\alpha}(\eta+\nu) - \mbox{LIF}_{\vartheta,\alpha}(\eta)\|_{A,\alpha} \leq \vartheta.
\end{equation}
\end{lemma}

Based on Theorem~\ref{th:quantization} on characterizing LIF as signal-to-spike-train quantization and taking into account the special cases 
of $\alpha \in \{0, \infty\}$ of Lemma~\ref{lem:IFnu} we obtain a Lipschitz-style upper bound in terms of inequality~(\ref{eq:LIFANInequality}) for a
{\it reset-to-mod} LIF neuron, resp. in terms of (\ref{eq:SNNBound}) for SNNs based on {\it reset-to-mod} LIF neurons.
\begin{theorem} [Liptschitz-Style Upper Bound for the LIF model]
\label{th:LIFDNInequality}
For a {\it reset-to-mod} LIF neuron model with $\vartheta>0$ and $\alpha\in [0,\infty]$ and for all spike trains 
$\nu \in \mathbb{S}$ there holds the following inequality
\begin{equation}
\label{eq:LIFANInequality}
\sup_{\eta \in \mathbb{S}}
\left\|\mbox{LIF}_{\vartheta, \alpha}(\eta + \nu) - \mbox{LIF}_{\vartheta, \alpha}(\eta)\right\|_{A, \alpha} 
\leq   \gamma(\alpha) 
\left\lceil
\frac{1}{\vartheta}\|\nu\|_{A, \alpha} 
\right\rceil \vartheta,
\end{equation}
where $\gamma(0)=\gamma(\infty)=1$ and $\gamma(\alpha) \in [2,3]$ for $\alpha \in (0, \infty)$.
\end{theorem}

\begin{proof}
First of all note that 
\begin{equation}
\label{eq:gammaConst}
\xi(\vartheta, \alpha):=\sup_{\eta,\nu \in \mathbb{S}} 
\frac{
\left\|\mbox{LIF}_{\vartheta, \alpha}(\eta + \nu) - \mbox{LIF}_{\vartheta, \alpha}(\eta)\right\|_{A, \alpha}}
{
\left\lceil
\|\frac{1}{\vartheta}\nu\|_{A, \alpha} 
\right\rceil\vartheta
}
\end{equation}
is independent from $\vartheta$
although $\vartheta$ appears in (\ref{eq:gammaConst}), as shown in the following. 
Indeed, for given threshold $\vartheta>0$ let
 $\eta_i^{(\vartheta)}$ and $\nu_i^{(\vartheta)}$ be sequences for which the fraction in~(\ref{eq:gammaConst})
converges to $\xi(\vartheta, \alpha)$, then 
$\widetilde{\eta_i} := \eta_i^{(\vartheta)}/\vartheta$ and 
$\widetilde{\nu_i} := \nu_i^{(\vartheta)}/\vartheta$ yield
\begin{equation}
\label{eq:gammaConst1}
\xi(\vartheta, \alpha) = 
\sup_i
\frac{
\vartheta \left\|\mbox{LIF}_{1, \alpha}(\widetilde{\eta_i} +\widetilde{\nu_i}) - \mbox{LIF}_{1, \alpha}(\widetilde{\eta_i})\right\|_{A, \alpha}}
{
\left\lceil
\|\frac{1}{\vartheta}\nu_i^{(\vartheta)}\|_{A, \alpha} 
\right\rceil\vartheta
} = \xi(1, \alpha).
\end{equation}

Now, define and use~(\ref{eq:quantization})
\begin{eqnarray}
\label{eq:gamma}
\gamma(\alpha) & := & \xi(1, \alpha) \nonumber \\
& = & 
\sup_{\eta, \nu \in \mathbb{S}} 
\frac{
\left\|\mbox{LIF}_{1, \alpha}(\eta + \nu) - (\eta + \nu) +  \nu +\eta-\mbox{LIF}_{1, \alpha}(\eta)\right\|_{A, \alpha}}
{
\left\lceil
\|\nu\|_{A, \alpha} 
\right\rceil
} \nonumber\\
& \leq  & 
\sup_{\eta, \nu \in \mathbb{S}} 
\frac{
2 + \|\nu\|_{A, \alpha}}
{
\left\lceil
\|\nu\|_{A, \alpha} 
\right\rceil
} \leq 3 < \infty.
\end{eqnarray}
Now, consider $\alpha \in (0, \infty)$ and the following example.
\begin{example}
\label{ex:1}
Let $\eta = \sum_{k=1}^3 a_k \delta_{t_k}$ for $t_k = k \varepsilon$, $\varepsilon>0$ 
and $(a_1,a_2,a_3)= (-\frac{3}{2}, 1, \frac{3}{2})$, 
and $\nu = \sum_{k=1}^3 b_k \delta_{t_k}$ with $(b_1,b_2,b_3)= (1, -1, 1)$ satisfying $\|\nu\|_{A, \alpha}=1$ for all $\alpha \in [0,\infty]$.
\end{example}
For $\alpha \in (0,\infty)$  we obtain for this example
\begin{eqnarray}
\mbox{LIF}_{1, \alpha}(\eta) & = & -1 \,\delta_{t_1} + 0 \,\delta_{t_2} + 1 \,\delta_{t_3}, \nonumber \\
\mbox{LIF}_{1, \alpha}(\eta + \nu) & = & 0 \,\delta_{t_1} + 0 \,\delta_{t_2} + 2 \,\delta_{t_3}.
\end{eqnarray}
Therefore we get 
\begin{eqnarray}
\label{eq:gammalowerbound}
\rho(\varepsilon, \alpha):= \|\mbox{LIF}_{1, \alpha}(\eta + \nu) -\mbox{LIF}_{1, \alpha}(\eta )\|_{A, \alpha} & = & 
\|1\, \delta_{t_1} + 0\, \delta_{t_2} + 1\, \delta_{t_3}\|_{A, \alpha}  \nonumber \\
& = &  \left|1+ e^{-2 \varepsilon \alpha}\right|.
\end{eqnarray}
From $\lim_{\vartheta \rightarrow 0} \rho(\varepsilon, \alpha) = 2$ for all $\alpha \in (0, \infty)$ 
we follow that $\gamma(\alpha)\geq 2$ for all $\alpha \in (0, \infty)$. 

Now, let us check the special case $\alpha = 0$. Without loss of generality we may assume that $\vartheta = 1$.
For this case we apply the spike train decomposition of Corollary~\ref{cor:ADecomposition}, which allows us to represent $\nu = \sum_{k=1}^a \nu_k + \widetilde{\nu}$,
where 
$a:= \left\lfloor \|\nu\|_{A,0} \right\rfloor$, 
$\|\nu_k\|_{A,0}=1$ and $\|\widetilde{\nu}\|_{A,0} < 1$.
Then, taking into account Lemma~\ref{lem:IFnu} and applying (\ref{eq:quantization}) on the telescope sum 
\begin{eqnarray}
\label{eq:tele}
 & & \|\mbox{LIF}_{1, 0}(\eta + \nu) -  \mbox{LIF}_{1, 0}(\eta )\|_{A,0}\nonumber  \\
& = & 
\|\, 
\mbox{LIF}_{1, 0}(\eta + \nu_1 + \ldots +\nu_a + \widetilde{\nu})  - \mbox{LIF}_{1, 0}(\eta + \nu_2 + \ldots +\nu_a + \widetilde{\nu}) \nonumber  \\
& &  
+\mbox{LIF}_{1, 0}(\eta + \nu_2 + \ldots +\nu_a + \widetilde{\nu}) - \mbox{LIF}_{1, 0}(\eta + \nu_3 + \ldots +\nu_a + \widetilde{\nu}) \nonumber  \\
& & \, \,\,\,\ldots \nonumber\\
& & +
\mbox{LIF}_{1, 0}(\eta + \nu_a + \widetilde{\nu})  - \mbox{LIF}_{1, 0}(\eta + \widetilde{\nu}) \nonumber  \\
& & +
\mbox{LIF}_{1, 0}(\eta + \widetilde{\nu}) - \mbox{LIF}_{1, 0}(\eta ) \,\|_{A,0} \nonumber \\
& \leq &  \left\lceil \|\nu\|_{A,0} \right\rceil
\end{eqnarray} 
we obtain $\gamma(0) \leq 1$.
Since Example~\ref{ex:1} gives
\begin{eqnarray}
\mbox{LIF}_{1, 0}(\eta) & = & -1 \,\delta_{t_1} + 0 \,\delta_{t_2} + 2 \,\delta_{t_3}, \nonumber \\
\mbox{LIF}_{1, 0}(\eta + \nu) & = & 0 \,\delta_{t_1} + 0 \,\delta_{t_2} + 2 \,\delta_{t_3},
\end{eqnarray}
hence, $\|\mbox{LIF}_{1, 0}(\eta + \nu) -\mbox{LIF}_{1, 0}(\eta)\|_{A, 0} = 1$, we finally get
$\gamma(0)=1$.
The same way of reasoning on the telescope sum applies to the case $\alpha = \infty$, giving $\gamma(\infty)\leq 1$.
Checking again Example~\ref{ex:1}, we get
\begin{eqnarray}
\mbox{LIF}_{1, \infty}(\eta) & = & -1 \,\delta_{t_1} + 1 \,\delta_{t_2} + 1 \,\delta_{t_3}, \nonumber \\
\mbox{LIF}_{1, \infty}(\eta + \nu) & = & 0 \,\delta_{t_1} + 0 \,\delta_{t_2} + 2 \,\delta_{t_3},
\end{eqnarray}
hence, $\|\mbox{LIF}_{1, \infty}(\eta + \nu) -\mbox{LIF}_{1, \infty}(\eta)\|_{A, 0} = \|1 \,\delta_{t_1} -1 \,\delta_{t_2} + 1 \,\delta_{t_3}\|_{A,0} = 1$, showing that
$\gamma(\infty)=1$.
\end{proof}

Theorem~\ref{th:LIFDNInequality} together with the triangle inequality of the norm $\|.\|_{A,\alpha}$  immediately yields 
a global upper bound on the norm difference of $\mbox{LIF}(\eta)$ and its perturbed version $\mbox{LIF}(\eta + \nu)$ for SNNs.
\begin{theorem} [Global Lipschitz-Style Bound for SNNs]
\label{th:SNNDNInequality}
Let the spiking neural network $\mbox{SNN}: \mathbb{S}^{N_0} \rightarrow \mathbb{S}^{(N_L)}$ with {\it reset-to-mod} LIF neurons $\mbox{LIF}_{\vartheta,\alpha}$, 
$\vartheta = 1$, be given by 
$[W^{(1)}, \ldots, W^{(N_L)}]$ according to~(\ref{eq:SNNW}), and let $(\nu_1, \ldots,\nu_{N_0})$ be additive error spike trains 
in the corresponding input spike trains $(\eta_1, \ldots,\eta_{N_0})$ , then for all output channels $\eta_j^{(N_L)}$, $j \in \{1, \ldots, N_L\}$, we obtain the following error bound 
\begin{eqnarray}
\label{eq:SNNBound}
& &  
\sup_{\eta_i}
\left\|
\mbox{SNN}\left((\eta_i + \nu_i)_i\right)-
\mbox{SNN}\left((\eta_i)_i \right)
\right\|_{A, \alpha} \nonumber\\
 & \leq_j &
\Gamma_{\alpha}\circ \widetilde{W}^{(N_L)} \left(\Gamma_{\alpha} \circ \widetilde{W}^{(N_L-1)} \cdots 
\left(\Gamma_{\alpha} \circ \widetilde{W}^{(1)}  \left(\Gamma_{\alpha} \circ  \nu_{A,\alpha}\right)\right)\right), 
\end{eqnarray}
where 
$\|(c_{i,j})_{i,j}\|_{A, \alpha}:= ((\|(c_{i,j})_{i,j} \|_{A, \alpha})_i)$, 
$\leq_j$ refers to the $j$-th output channel on the left and the right hand side of the inequality, 
$\widetilde{W}^{(k+1)}:=\left(\left|w^{(k+1)}_{i_{k+1}, i_k}\right|\right)_{i_{k+1}, i_k}$, 
$\nu_{A,\alpha}  :=  \left(\left\|\nu_1\right\|_{A, \alpha}, \ldots,\left\|\nu_{N_0}\right\|_{A, \alpha}\right), 
\Gamma_{\alpha}(x) :=  \left\lceil \gamma(\alpha) \, x\right\rceil$
and the rounding-up function $ \left\lceil .\right\rceil$  is applied coordinate-wise.
\end{theorem}

\section{Evaluation}
\label{s:Evaluation}
In this section we look at numerical examples to demonstrate the  main theoretical results of our paper, that is above all (\ref{eq:quantization})
on spike train quantization and its consequences in terms of 
quasi isometry, Corollary~\ref{cor:QuasiIsometry}, the error bounds w.r.t time delay, Corollary~\ref{cor:BoundonLag} 
and the global Lipschitz-stlye upper bound for additive spike trains due to Theorem~\ref{th:LIFDNInequality} for the LIF model, resp. Theorem~\ref{th:SNNDNInequality} for LIF-based feedforward SNNs. 
See~\url{https://github.com/LunglmayrMoser/AlexSNN} for Python and Mathematica code. 

All the theoretical findings of this paper are based on the choice of {\it reset-to-mod} as re-initialization mode. Therefore, in the subsequent evaluations we also take the other reset modes into account to get an overview about 
the differences in the behavior. 
It is also instructive to look at the effect of alternative distance measures that are not equivalent to the Alexiewicz norm. We restrict this comparison to the Euclidean based norm~(\ref{eq:L2norm}). 
Other metrics for spike trains such as ~\cite{Satuvuori2018,Sihn2019,VICTOR2005585} are not considered because they  
are motivated for other purposes than considered in our approach which aims at characterizing that topology of the vector space $\mathbb{S}$ which meets the Postulates 1 and 2 of Section~\ref{s:Postulates}. Therefore a detailed discussion on potential implications of the Alexiewicz topology (and its leaky variants)  in the context of other proposed metrics is postponed for future study. 

\subsection{Spike Train Quantization due to~(\ref{eq:quantization})}
\label{sss:Quantization}
Fig.~\ref{fig:QuantResetVariants} displays the quantization error in the Alexiewicz topology, i.e., 
$\|\mbox{LIF}_{\vartheta, \alpha}(\eta)-\eta\|_{A, \alpha}$-norm for the different reset variants: 
(a) {\it reset-to-zero}, (b) {\it reset-by-subtraction}, and, ours, (c) {\it reset-to-mod}.
For large leaky parameter $\alpha$ all three variants tend to same error behavior. 
As expected according to Theorem~\ref{th:quantization} only in the {\it reset-to-mod} variant we can guarantee the bound of of Theorem~\ref{th:quantization}.
In contrast, Fig.~\ref{fig:QuantResetVariantsL2} illustrates the effect of choosing the Euclidean topology as commonly used in the context of SNNs, i.e., 
\begin{equation}
\label{eq:L2norm}
\left\|\sum_{i=1}^N a_i \delta_{t_i}\right\|_{2, \alpha} := \sqrt{\sum_{k=1}^N \left(\sum_{i=1}^k a_i e^{-\alpha (t_k-t_i)}\right)^2}.
\end{equation}
In contrast to the Alexiewicz topology and the {\it reset-to-mod} re-initialization of the LIF neuron 
there is no guarantee in the Euclidean metric to have a global upper bound for the quantization error, see 
Fig.~(\ref{fig:L2vsA}).
\begin{figure}
\centering
\includegraphics[height=3cm]{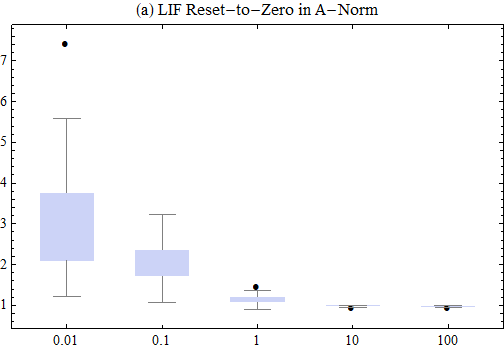}
\includegraphics[height=3cm]{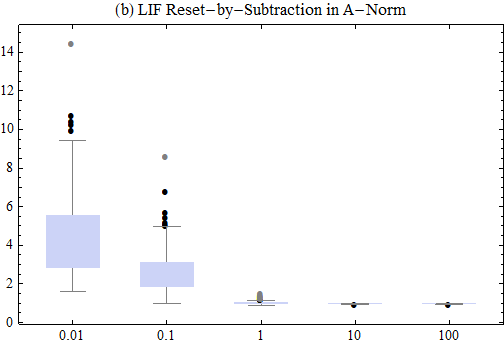}
\includegraphics[height=3cm]{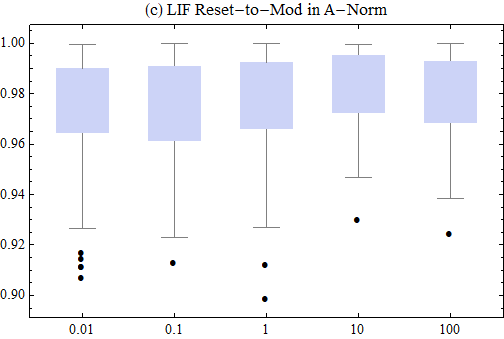}
  \caption{Distribution of the quantization error $\|\mbox{LIF}_{\vartheta, \alpha}(\eta)-\eta\|_{A, \alpha}$ for $\vartheta=1$, $\alpha \in \{0.01, 0.1, 1, 10, 100\}$ and the reset variants: (a) {\it reset-to-zero}, (b) {\it reset-by-subtraction}, and, (c) {\it reset-to-mod} (ours); 
	the spike trains with $50$ spikes (at equidistant grid) are generated by uniformly distributed spike amplitudes in the range of $[-2,2]$; for each variant $100$ runs are performed.}
		\label{fig:QuantResetVariants}
\end{figure}
\begin{figure}
\centering
\includegraphics[height=3cm]{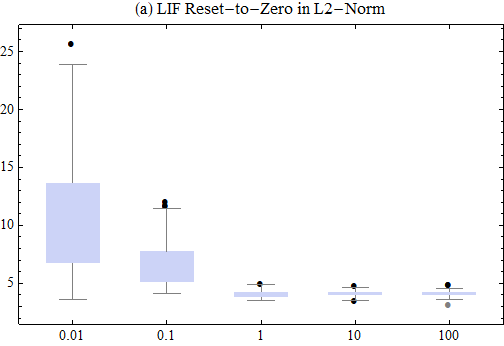}
\includegraphics[height=3cm]{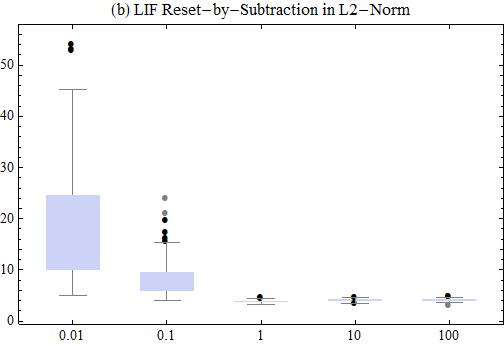}
\includegraphics[height=3cm]{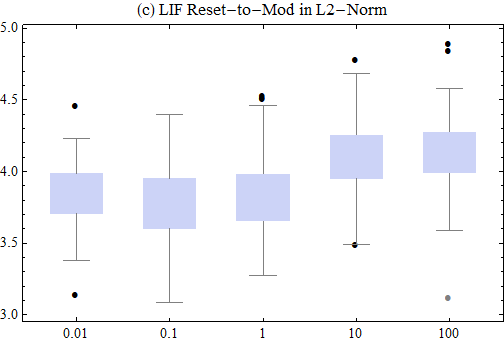}
  \caption{Like in Fig.~\ref{fig:QuantResetVariants} with $L_2$-based norm~(\ref{eq:L2norm}).}
		\label{fig:QuantResetVariantsL2}
\end{figure}
\begin{figure}
\centering
\includegraphics[height=3cm]{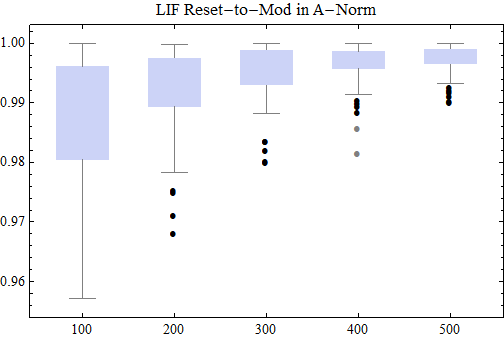}
\includegraphics[height=3cm]{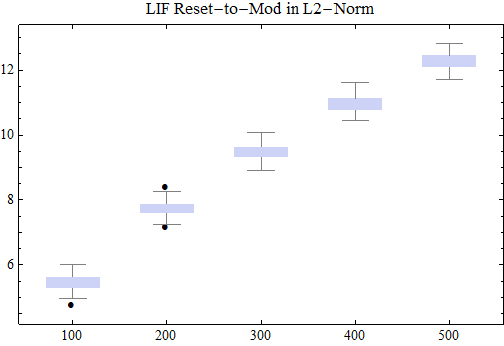}
  \caption{Like Fig.~\ref{fig:QuantResetVariants} and~\ref{fig:QuantResetVariantsL2} for $\alpha=1$ but different numbers of spikes, 
	$N \in \{100, \ldots, 500\}$. While the quantization error in the Euclidean norm~(\ref{eq:L2norm}) increases with $N$ (right), due to Theorem~\ref{th:quantization} it remains strictly upper bounded by the threshold in the Alexiewicz topology (left).
	Note the concentration of measure effect in the Alexiewicz topology, 
 see~\cite{Vershynin2018}.
	}
		\label{fig:L2vsA}
\end{figure}

\subsection{Error Bounds regarding Postulates 1 and 2 and Quasi Isometry}
\label{sss:Quasi}
Postulates 1 and 2 are covered by the inequalities~(\ref{eq:lag1}), resp. (\ref{eq:BoundonThetaPerturbation}), regarding time delay, resp. threshold deviation.
Fig.~\ref{fig:EvalPostulates} shows an example including the theoretical upper bound proven for the {\it reset-to-mod} variant together with the
norm-errors resulting from our three re-initialization variants and different settings of the leakage parameter. 
For time delays the theoretical upper bound is guaranteed 
for sufficiently small time delays (see~\ref{A:Lag}). 
In this example the other re-initialization variants {\it reset-by-subtraction} and  {\it reset-to-zero} show smaller errors compared to {\it reset-to-mod} .
For threshold deviations it is the other way round and as guaranteed by~(\ref{eq:BoundonThetaPerturbation}) the {\it reset-to-mod} related dashed red line is strictly below the bound (black line) for all $\Delta \vartheta$.
\begin{figure}
\centering
\includegraphics[width=11cm]{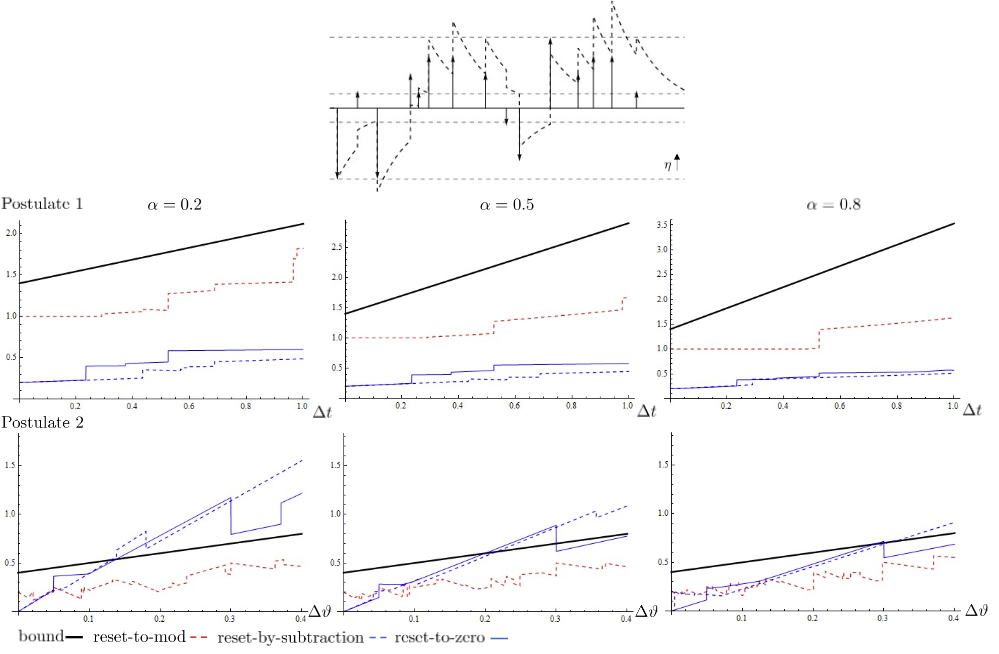}
\caption{Evaluation of the effect of time delay $\Delta t$ (second row) and threshold deviations $\Delta \vartheta$ (third row) for a single LIF neuron for different $\alpha \in \{0.2, 0.5, 0.8\}$.
  The plots show the left side of the inequalities~(\ref{eq:lag1}), resp.~(\ref{eq:BoundonThetaPerturbation}), for the three reset variants together with the bound (black line) given by the right side of the corresponding inequalities. 
	} 
\label{fig:EvalPostulates}
\end{figure}

Like the data for spike train quantization, also the analysis of quasi isometry, see Fig.\ref{fig:QuasiIsometry}, shows significant
differences regarding the choice of the re-initialization mode.
\begin{figure}
\centering
\includegraphics[width=11.5cm]{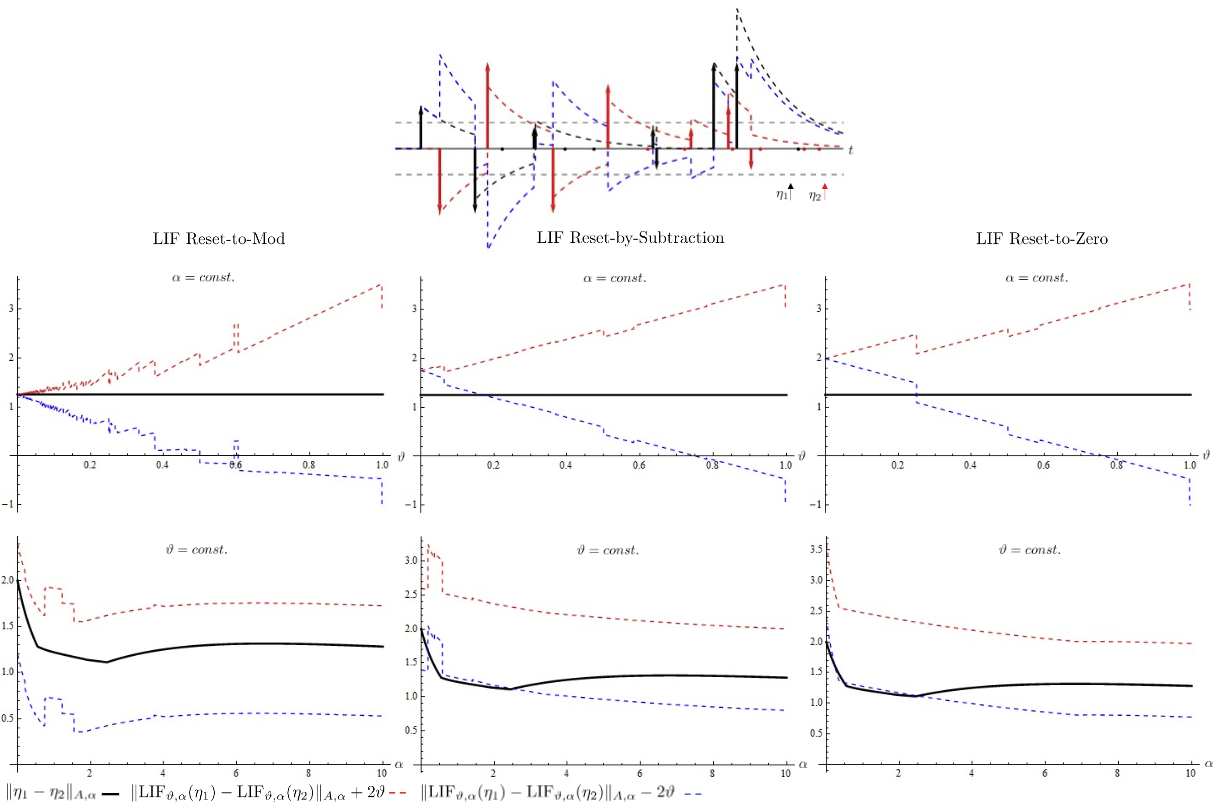}
  \caption{Evaluation of quasi isometry.  
Second and third row: Evaluation of $\|.\|_{A, \alpha}$-norm of   
$\eta_1 - \eta_2$ after applying $\mbox{LIF}_{\vartheta, \alpha}$
for const. $\alpha = 4$, second row, and const. $\vartheta = 0.3$ for third row.
Only 
{\it reset-to-mod} meets the conditions of quasi isometry due to~(\ref{eq:quIsometry}).}
\label{fig:QuasiIsometry}
\end{figure}

\subsection{Lipschitz-Style Upper Bound for LIF and SNNs}
\label{ss:Resonance}
First we look at a single LIF neuron. Theorem~\ref{th:LIFDNInequality} actually addresses two aspects. 
First,  the global Lipschitz-style bound, and second, the observation that the upper bound constant $\gamma(\alpha)$ 
behaves different for $\alpha \in \{0, \infty\}$ and $\alpha \in (0, \infty)$. 
In Fig.~\ref{fig:ResonanceExp1} we compute this effect for different leakage parameters $\alpha$ and scaling factors $\lambda$ of the additive spike train $\nu$.
As these $(\alpha, \lambda)$-plots show the amplification factor can be quite discontinuous and jagged.
Different interfering $\nu$ signals can cause quite different shapes.
The more it is remarkable that the amplification factor is globally bounded for all input spike trains irrespective of how many spikes they contain, and that
for $\alpha = 0$ (or, large)  we have the tight bound of $\gamma(\alpha)=1$. 
For $\alpha \in (0,\infty)$ we find examples converging to $\gamma(\alpha)=2$ as proven in the Theorem. 
It is a conjecture that in fact $\gamma(\alpha)=2$, though in the proof we only have evidence that $\gamma(\alpha) \in [2,3]$.
It remains an open question to analyze this resonance-type phenomenon in more detail. 
In contrast, the different shapes in the 3D plots comparing the re-initialization mode {\it reset-to-mod} (first row of 3D plots) with that of {\it reset-to-zero}
can be explained more easily. Since a large $\alpha$ reduces the dependence on spikes in the past, hence approximating the behavior of {\it reset-to-zero}.
\begin{figure}
\centering
\includegraphics[width=13cm]{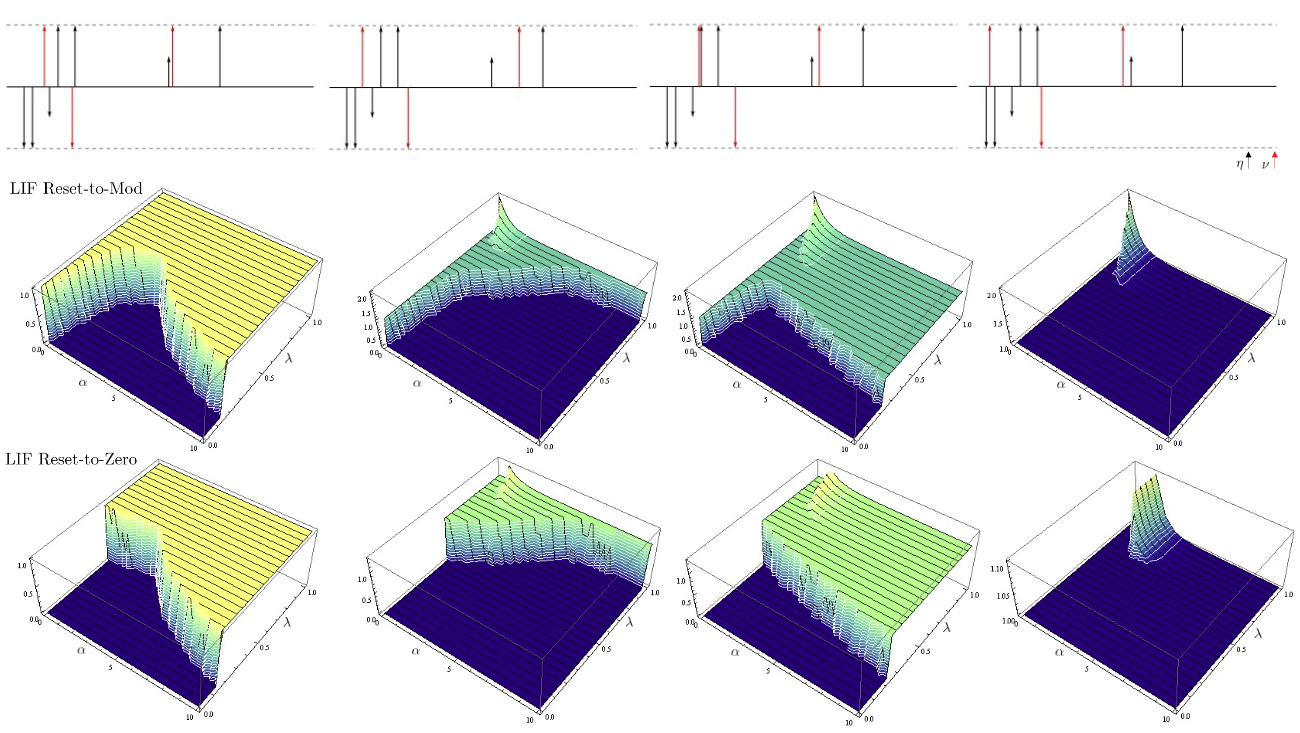}
  \caption{$(\alpha, \lambda)$-plot evaluations of the left hand side of inequality~(\ref{eq:LIFANInequality}) for 
	four variations of Example~\ref{ex:1} for $\alpha \in [0, 10]$ in the x-axis and scaling factor $\lambda \in [0,1]$ in the y-axis. 
			}
		\label{fig:ResonanceExp1}
\end{figure}

($\alpha$, $\lambda$)-plots of Fig.~\ref{fig:ResonanceExp1} look similar for SNNs. 
Due to Theorem~\ref{th:SNNDNInequality} they are globally bounded for all input spike trains.
However, the shape can be quite jagged and discontinuous as illustrated by Fig.~\ref{fig:SNNBoundExample} which shows an example for a $3$-layered SNN with
\begin{equation}
\label{eq:SNNexp}
W^{(1)} = \left( 
\begin{array}{cc}
	1 & 1 \\
	1 & 2
\end{array}
\right), \, 
W^{(2)} = \left( 
\begin{array}{cc}
	0.5 & 0 \\
	0.5 & 0.5 \\
	0   & -0.5
\end{array}
\right), \,
W^{(3)} = \left( 
\begin{array}{ccc}
	1 & 1 & 1
\end{array}
\right).
\end{equation}
The comparison of the two examples in Fig.~\ref{fig:SNNBoundExample} show the sensitivity of time. 
After shifting the disturbing red spike to the green position the resulting $(\alpha, \lambda)$-plot breaks the symmetry 
causing a different characteristic of the shape. A detailed analysis of these effects is postponed to future research.  
\begin{figure}
\centering
\includegraphics[width=13cm]{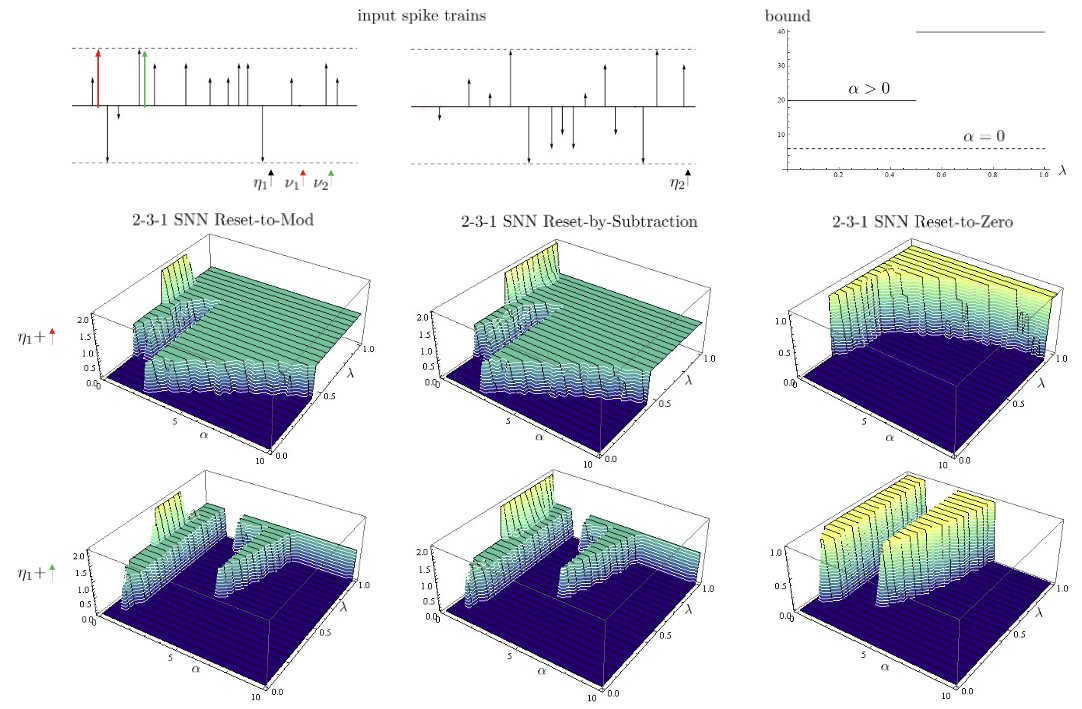}
  \caption{($\alpha$, $\lambda$) evaluation of the left-hand side of the inequality~(\ref{eq:SNNBound}) for 
	the $2$-$3$-$1$ SNN given by the weight matrices~(\ref{eq:SNNexp}) and the input spike trains $\eta_i$. 
		 The outer right graph depicts the right-hand side of the inequality~(\ref{eq:SNNBound}) where the cases $\alpha=0$ and 
	$\alpha >0$ are distinguished.
	The second row shows the ($\alpha$, $\lambda$) evaluation for the red additional spike, 
	while the second row shows the evaluation with the green additional spike.  
	In the the second row we disturb $\eta_1$ by the red spike, and accordingly, in the third row we add the green spike and 
	the resulting 3D-plots of the measured error in the Alexiewicz norm like in Fig.~\ref{fig:ResonanceExp1} for 
the three different re-initialization modes as indicated.
		}
		\label{fig:SNNBoundExample}
\end{figure}

\section{Outlook and Conclusion}
\label{s:Conclusion}
Our approach starts with the well-known observation that bio-inspired signal processing leads to
 a paradigm shift in contrast to the well-established technique of clock-based sampling and processing.
Driven by the hypothesis that this paradigm shift must also manifest itself in its mathematical foundation, 
we started our analysis in terms of a top-down theory development by first searching for informative postulates.
For the LIF neuron model (and SNNs based thereupon) 
our analysis shows that there is an underlying non-Euclidean geometry 
that governs its input-output behavior. As the central result of this paper, it turns out that this mapping can be fully characterized as
signal-to-spike train quantization in the Alexiewicz norm, resp. its adaptation for a positive leakage parameter.
While we gave a proof for this result for spike trains represented by a sequence of weighted Dirac impulses of arbitrary time 
intervals, we indicated in the Appendix that this quantization principle also holds for a wider class. 
Going beyond that, our conjecture is that the quantization error inequality will hold for all signals for which the formula is well-defined, but, 
that for being able to achieve this one will resort to an alternative concept of integration, namely the Henstock-Kurzweil integral which is related to the
Alexiewicz topology. This remains to be worked out in follow-up research. So does the analysis of the resonance-like phenomenon of the  in the context of 
the Lipschitz-style error bound. Another research direction is to explore the potential of our Alexiewicz norm-based approach for information coding and, more generally, for establishing a unified theory that incorporates low-level signal acquisition through event-based sampling and signal processing via feedback loops and learning strategies for high-level problem solving. 
Another thread running through the paper is the question of the choice and impact of the re-initialization variant. 
The quantization theorem is stated for the variant, which we coined {\it reset-to-mod} and results from applying {\it reset-by-subtraction} instantaneously in the case of spike amplitudes that exceed the threshold by a multiple. 
The resulting found properties such as quasi isometry or the error bound on time delay might be arguments for {\it reset-to-mod}, but in the end
this study can only be seen as a starting point towards a more comprehensive theoretical foundation of bio-inspired signal processing taking 
its topological peculiarities into account.
\section*{Acknowledgements}
This work was supported (1) by the 'University SAL Labs'
initiative of Silicon Austria Labs (SAL) and its Austrian partner universities for applied fundamental research for electronic based systems, (2) by Austrian ministries BMK, BMDW, and the State of UpperAustria in the frame of SCCH, part of the COMET Programme managed by FFG, and (3) by the 
{\it NeuroSoC} project funded under the Horizon Europe Grant Agreement number 101070634.

\appendix
\section{Proof of Theorem~\ref{th:closureIF}}
\label{ss:th:closureIF}
We show the proof for $\alpha = 0$. For $\alpha >0$ the argumentation is analogous.

From left to right. 
Consider a sequence $\eta_n$ of sub-threshold spike trains $\eta_n = \sum_{i_n} a_{i_n}^{(n)} \delta_{t_{i_n}}$, 
then an integrate-and-fire neuron never reaches the threshold $\theta>0$, i.e., for all $n$ and $m$ we have 
$|\sum_{i_n=0}^m a_{i_n}^{(n)}| < \theta$. Consequently, taking the limit w.r.t. $n$ we obtain the right-hand side of 
Equ.~(\ref{eq:closure}).

From right to left. Now, consider $\eta = \sum_i a_i\delta_{t_{i}}$ satisfying the inequality of the right-hand side of  
Equ.~(\ref{eq:closure}). If there is no spike in the output, the spike train is sub-threshold, i.e., $\eta \in C$, hence $\eta \in \overline{C}$. Assume that there is at least one spike in the output. Without loss of generality, let us assume that the first spike is positive.
Then we define $i_0 := 0$ and $i_k$ recursively by $i_{k+1} := \min\{j: \sum_{i = i_k+1}^{j} = (-1)^k 2 \vartheta\}.$
Note that $a_i^{(\varepsilon)} := a_i - (-1)^k \varepsilon a_i/ \vartheta\}$ 
yields a spike train $\eta^{(\varepsilon)} \in C$ that converges to $\eta$. 

\section{Unit Ball of $\|.\|_{A,0}$}
\label{A:normA}
In this section we characterize the unit ball $B_A$ of $\|.\|_A$ as sheared transform of the hypercube $[-1,1]^N$, i.e.,
$B_A = \{x \in \mathbb{R}^N:\, \|x\|_A \leq 1\} = \{x = Ty :\, y \in [-1,1]^N\}$, where
\begin{equation}
\label{eq:T}
T = \left(
\begin{array}[4]{ccccc}
	1   & 0 & \cdots & \cdots &0 \\
	-1  & 1 & 0 & \cdots & 0 \\
	\vdots & \cdots & \ddots & \ddots & \vdots \\
	0 & \cdots & \cdots & -1 & 1
\end{array}
\right).
\end{equation}

\begin{proof}
We are interested to characterize all $x=(x_1, \ldots, x_N) \in \mathbb{R}^N$ such that $\|x\|_A\leq 1$, i.e.,
$y_n:= \sum_{i=1}^n x_i \in [-1,1]$ for all $n \in \{1, \ldots, N\}$.
Expressing $x_i$ in terms of $y_i$ means $y_1 = x_1$, $y_2 - y_1  =  x_2$, $\ldots$, $y_N - y_{N-1} = x_N$, that means $x = Ty$ due to (\ref{eq:T}).
\end{proof}

\section{Proof for Spike Train Quantization for Dirac Impulses}
\label{A:Dirac:th:quantization}
We recall the proof from~\cite{MoserLunglmayrESANN2023}
\begin{theorem}[{\it \bf reset-to-mod} LIF Neuron as $\|.\|_{A,\alpha}$-Quantization]
\label{th:quantizationDirac}
Given a LIF neuron model with  {\it reset-to-mod}, the LIF parameters $\vartheta>0$ and $\alpha \in [0,\infty]$ and the spike train 
$\eta \in \mathbb{S}$ with amplitudes $a_i \in \mathbb{R}$. 
Then, $\mbox{LIF}_{\vartheta, \alpha}(\eta)$ is a 
$\vartheta$-quantization of $\eta$, i.e., the resulting spike amplitudes 
are multiples of $\vartheta$, where the quantization error is bounded by (\ref{eq:quantization}),
hence 
$\mbox{LIF}_{\vartheta, \alpha}(\mbox{LIF}_{\vartheta, \alpha}(\eta) -\eta) = \emptyset$ and 
$\mbox{LIF}_{\vartheta, \alpha}\left(\mbox{LIF}_{\vartheta, \alpha}(\eta)\right) = \mbox{LIF}_{\vartheta, \alpha}(\eta)$.
\end{theorem}

\begin{proof}
\begin{wrapfigure}{r}{7cm}
  \includegraphics[width=6.5cm]{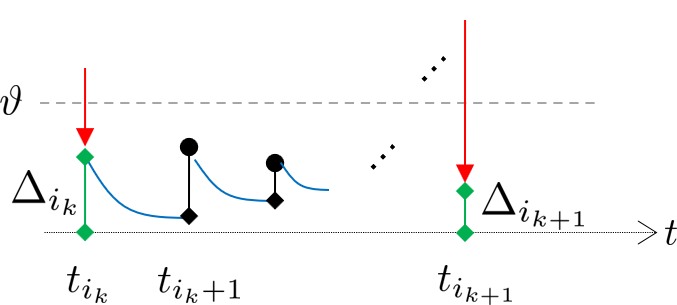}
	\caption{Illustration of Equation~(\ref{eq:quantDeltaRecursion}). The red arrows indicate {\it reset-by-mod}.}
 		\label{fig:quantProof}
\end{wrapfigure}

First of all, we introduce the following operation $\oplus$, which is associative and can be handled with like the usual addition if adjacent elements $a_i$ from a spike train $\eta = \sum_i a_i \delta_{t_i}$ are aggregated:
\begin{equation}
\label{eq:pseudoaddition}
a_i \oplus a_{i+1} :=  e^{-\alpha (t_{i+1}- t_{i})} a_i  +  a_{i+1}.
\end{equation}
This way we get a simpler notation when aggregating convolutions, e.g., 
\[
a_i \oplus \ldots \oplus a_j = \sum_{k=i}^j e^{-\alpha (t_{j}- t_{k})} a_k.
\] 

For the discrete version we re-define 
$a_{i_k} \oplus a_{i_{k+1}} :=  \beta^{(i_{k+1}- i_{k})} a_{i_k}  +  a_{i_{k+1}}$, if $i_{k}$ and $i_{k+1}$ refer to adjacent spikes at time $i_k$, resp. $i_{k+1}$. Further, we denote $q[z]:= \mbox{sgn}(z)\, \left\lfloor |z|\right\rfloor$ which is the ordinary quantization due to integer truncation, e.g. $q[1.8] = 1$, $q[-1.8]=-1$, where $\left\lfloor |z| \right\rfloor = \max\{n\in \mathbb{N}_0:\, n \leq |z|\}$.  

After fixing notation let us consider a spike train $\eta = \sum_j a_i \delta_{t_{i}}$. 
Without loss of generality we may assume that $\vartheta = 1$.
We have to show that $\|\mbox{LIF}_{1, \alpha}(\eta)  - \eta\|_{A, \alpha} < 1$, which is equivalent 
to the discrete condition that $\forall n: \max_n \left|\sum_{i=1}^n \hat{a}_i \right| < 1$,
where  $\eta - \mbox{LIF}_{1, \alpha}(\eta) = \sum_i \hat{a}_i \delta_{t_i}$.
Set $\hat{s}_k := \hat{a}_0 \oplus \cdots \oplus \hat{a}_k$. We have to show that
$\max_{k}|\hat{s}_k| < 1$. The proof is based on induction and leads the problem back to the standard quantization by truncation.

Suppose that at time $t_{i_{k-1}}$ after re-initialization by {\it reset-to-mod} we get the residuum $\Delta_{i_{k-1}}$ as membrane potential that is the starting point for the integration after $t_{i_{k-1}}$. 
Note that 
\begin{equation}
\mbox{LIF}_{1, \alpha}(\eta)|_{t = t_k} = q(\Delta_{i_{k-1}} \oplus a_{i_{k-1}+1} \cdots \oplus a_{i_{k}}) \nonumber
\end{equation}
Then, as illustrated in Fig.~\ref{fig:quantProof} the residuum $\Delta_{i_{k}}$ at the next triggering event $t_{i_{k}}$ is obtained by the equation
\begin{equation}
\label{eq:quantDeltaRecursion}
\Delta_{i_{k}} = \Delta_{i_{k-1}} \oplus a_{i_{k-1}+1} \oplus \ldots \oplus  a_{i_k} - q[\Delta_{i_{k-1}} \oplus \ldots \oplus  a_{i_k}].
\end{equation}
Note that due to the thresholding condition of LIF we have
\begin{equation}
\label{eq:thcond}
|\Delta_{i_{k}} \oplus a_{i_{k}+1} \oplus \ldots \oplus  a_j| < 1
\end{equation}
for $j \in \{i_{k}+1, \ldots, i_{k+1}-1\}$.
For the $\oplus$-sums $\hat{s}_{i_k}$ we have
\begin{equation}
\label{eq:ahat}
\hat{s}_{i_{k+1}} = \hat{s}_{i_{k}} \oplus a_{i_{k}+1} \cdots a_{i_{k+1}-1} \oplus 
\left( 
a_{i_{k+1}} - q[\Delta_{i_{k}} \oplus a_{i_{k}+1} \oplus \ldots \oplus a_{i_{k+1}}]
\right).
\end{equation}

Note that $\hat{s}_0 = \Delta_{i_0}= a_0 - q[a_0]$, then for induction we assume that up to index $k$ to have
\begin{equation}
\label{eq:quantInduction}
\hat{s}_{i_k} = \Delta_{i_k}.
\end{equation}

Now, using (\ref{eq:quantInduction}), Equation~(\ref{eq:ahat}) gives
\begin{eqnarray}
\label{eq:induction}
\hat{s}_{i_{k+1}} & = & \Delta_{i_k} \oplus a_{i_{k}+1} \oplus \ldots \oplus 
a_{i_{k+1}-1} \oplus 
\left( 
a_{i_{k+1}} - q[\Delta_{i_{k}} \oplus a_{i_{k+1}} \oplus \ldots \oplus a_{i_{k+1}}]
\right) \nonumber \\
& = & \Delta_{i_k} \oplus a_{i_{k}+1} \oplus \ldots \oplus 
a_{i_{k+1}-1} \oplus 
a_{i_{k+1}} - q[\Delta_{i_{k}} \oplus a_{i_{k+1}} \oplus \ldots \oplus a_{i_{k+1}}], \nonumber\\
 & = & \Delta_{i_{k+1}}
\end{eqnarray}
proving (\ref{eq:quantInduction}), which together with (\ref{eq:thcond}) ends the proof showing that 
$|\hat{s}_k|<1$ for all $k$.

\end{proof}

\section{Example for Spike Train Decomposition}
\label{A:Decomposition}
Algorithm~\ref{alg:Decomposition} summarizes this approach in the proof of Theorem~\ref{cor:ADecomposition}
and Fig.~\ref{fig:Decomposition} 
demonstrates an example.
\begin{algorithm}
\caption{Spike Train Decomposition}
\label{alg:Decomposition}
\noindent
{\bf Step 0}:  Initialization: $\eta_0 := \eta$, $r=0$;
\\
\noindent
{\bf Step 1}: Up and Down Intervals: Partition the time domain into up and down intervals $\overline{J}_k$, resp.  $\underline{J}_k$ due to Equ.~(\ref{eq:Jk}) based on the 
top and bottom peaks positions $\overline{m}_k$, resp., $\underline{m}_k$ in the resulting walk according to Equ. (\ref{eq:topPeak}), resp. (\ref{eq:bottomPeak}).
\noindent\\
{\bf Step 2}:  Unit Discrepancy Delta: Define $\Delta \eta_{r+1}$ according to (\ref{eq:a1}), (\ref{eq:d1}), (\ref{eq:d2}), (\ref{eq:d11}) and (\ref{eq:d21}).
\\
{\bf Step 3:}  Subtraction:	$\eta_{r+1}:= \eta_r - \Delta\eta_{r+1}$;
\\
{\bf Step 4:}  Repeat steps {\bf 1}, {\bf 2} and {\bf 3} until $r=\|\eta\|_{A,0}$ to obtain $\eta = \sum_{k=1}^r \Delta\eta_k$.
\end{algorithm}

\begin{figure}
\centering
\includegraphics[width=13cm]{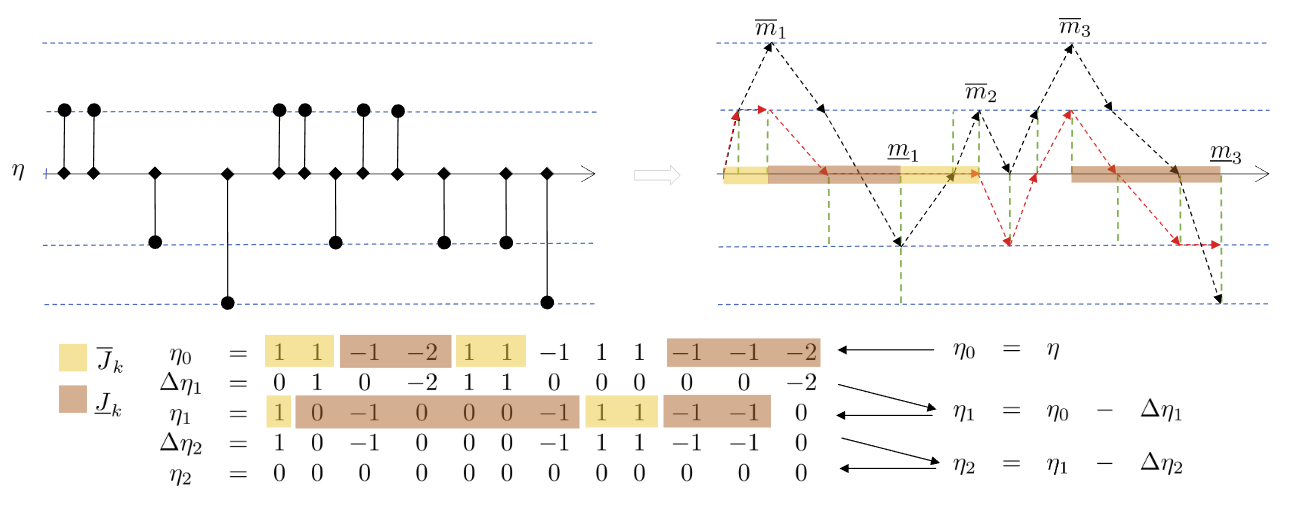}
  \caption{Example of decomposing a spike train $\eta$ with $\|\eta\|_{A,0}=2$ into a sum 
	$\eta = \sum_{k=1,2} \Delta \eta_k$ with $\|\Delta \eta_k\|_{A,0} = 1$, according to Algorithm~\ref{alg:Decomposition}.
	Top left: spike train $\eta$; Top right: illustration of first step with resulting walk and its top and bottom peaks due to (\ref{eq:topPeak}) and (\ref{eq:bottomPeak}).
	The red dashed line marks the resulting walk after subtracting $\Delta \eta_1$.
	Bottom: Recursive steps with highlighted up and down intervals according to (\ref{eq:Jk}).
	}
		\label{fig:Decomposition}
\end{figure}

\section{Error Bound on Lag, Corollary~\ref{cor:BoundonLag}}
\label{A:Lag}
The quasi-isometry property (\ref{eq:quIsometry}) yields
\begin{eqnarray}
\label{eq:lagproof1}
\left\| 
\mbox{LIF}_{\vartheta, \alpha}(\eta(\cdot - \Delta t)) -
\mbox{LIF}_{\vartheta, \alpha}(\eta)   
\right\|_{A, \alpha} & \leq &  
\left\|\eta(\cdot - \Delta t) -\eta \right\|_{A, \alpha} + 2\vartheta.
\end{eqnarray}
Because of $e^{\alpha \Delta t} = 1 + \alpha \Delta t + O(\Delta t^2)$ for $\Delta t \approx 0$, we get
\begin{eqnarray}
\label{eq:lagproof2}
\left\|\eta(\cdot - \Delta t) -\eta \right\|_{A, \alpha} & = & 
\max_{i} 
\left| 
\sum_{j=1}^{i-1} 
\left( 
a_j e^{-\alpha (t_i - t_j)} - a_j e^{-\alpha (t_i - t_j- \Delta t)} 
\right) + a_i
\right| \nonumber \\
 & = & 
\max_{n} 
\left| 
\sum_{j=1}^{n-1} 
a_j e^{-\alpha (t_n - t_j)} \left( 1- e^{\alpha \Delta t)} \right) + a_n
\right| \nonumber \\
& = & 
\max_{i} 
\left|
(-\alpha \Delta t)\sum_{j=1}^{i-1} a_j e^{-\alpha (t_i - t_j)}  + a_n 
\right| + O(\|\eta\|_{A, \alpha} \Delta t) \nonumber \\
& \leq &
\alpha \, (\|\eta\|_{A, \alpha} +\max_i |a_i|) \, \Delta t + \max_i |a_i|,
\end{eqnarray}
which together with (\ref{eq:lagproof1}) proves (\ref{eq:lag1}). 

\section{Proof of Lemma~\ref{lem:IFnu}}
\label{A:lem:IFnu}
First, let us check the case $\alpha =0$. Without loss of generality we may assume that $\vartheta = 1$.
Given spike trains $\eta, \nu \in \mathbb{S}$, $\eta = \sum_i a_i \delta_{t_i}$ and 
$\nu = \sum_i b_i \delta_{t_i}$ and suppose that $\|\nu\|_{A,0}\leq 1$.
Denote by $p_{\eta}(t_i^-)$ the membrane's potential in the moment before triggering a spike w.r.t the input spike train $\eta$.

Indirectly, let us assume that there are three subsequent spikes with the same polarity (all negative or all positive) generated by adding $\nu$ to $\eta$.
Without loss of generality we may assume that the polarity of these three spike events is negative.
Let denote these time points by $t_{s_1} <  t_{s_2} < t_{s_3}$ and let us consider the time point $t_{s_0} \leq t_{s_1}$ at which for the first time $\nu$ contributes to the negative spike event at $t_{s_1}$. Then, the first spiking event after $t_{s_0}$ is realized at $t_{s_1}$, which is characterized by
\begin{equation}
\label{eq:spikes1}
p_{\eta}(s_1)- \sum_{i = s_0}^{s_1} b_i < [p_{\eta}(s_1)] - 1.
\end{equation}
After the spike event at $t_{s_1}$ the re-initialization due to {\it reset-to-mod} takes place, meaning the reset of the membrane's potential increase by 
$1$, resulting in the addition of $(1 - \sum_{i = s_0}^{s_1} b_i)$ to the original membrane's potential $p_{\eta}(s_1)$.
Hence, we obtain for the second subsequent spike at $t_{s_2}$ the firing condition
\begin{equation}
\label{eq:spikes2}
p_{\eta}(s_2) + (1 - \sum_{i = s_0}^{s_1} b_i) - \sum_{i > s_1}^{s_2} b_i < [p_{\eta}(s_2)] - 1,
\end{equation}
thus,
\begin{equation}
\label{eq:spikes21}
p_{\eta}(s_2) -[p_{\eta}(s_2)]  +2  < \sum_{i = s_0}^{s_2} b_i.
\end{equation}
Analogously, we obtain as characterizing condition for the third spiking event
\begin{equation}
\label{eq:spikes21}
p_{\eta}(s_3) -[p_{\eta}(s_3)]  + 3  < \sum_{i = s_0}^{s_3} b_i.
\end{equation}
Since $\|\nu\|_{A, 0} \leq 1$, i.e., $|\sum_{i=1}^k b_i| \leq 1 $  for all $k$ it follows that 
$|\sum_{i=k_1}^{k_2} b_i| = |\sum_{i=1}^{k_2} b_i  - \sum_{i=1}^{k_1-1} b_i| \leq 2$.
Which applied to (\ref{eq:spikes21}) yields the contradiction
\begin{equation}
\label{eq:spikes3contra}
p_{\eta}(s_3) -[p_{\eta}(s_3)]  +3  < \sum_{i = s_0}^{s_3} b_i \leq 2, 
\end{equation}
namely, $-1 < p_{\eta}(s_3) -[p_{\eta}(s_3)]  \leq -1$. Therefore, there are at most $2$ subsequently triggered additional spikes by adding $\nu$.

The same way of reasoning, now with Equ.~(\ref{eq:spikes2}), shows that the first occurrence of an added spike can only be a single event which if any has to be followed by a spike event with different polarity. All together this means that 
$\|\mbox{LIF}_{1,0}(\eta+\nu)\|_{A,0} - \mbox{LIF}_{1,0}(\eta)\|_{A,0} \leq 1$. 
\\
\noindent
The second case of $\alpha=\infty$, i.e., $\|\nu\|_{A, \infty} = \max_i |b_i| \leq 1$ reduces to the standard quantization of integer truncation, i.e., to show
that $\max_i |[a_i +  b_i] - [a_i]|\leq 1$, which follows from
the fact that $-1 \leq a_i - [a_i] + b_i < 2$ for positive $a_i$ and $-2 < a_i - [a_i] + b_i \leq 1$ for negative $a_i$.

\bibliographystyle{unsrtnat}
\bibliography{references}  

\end{document}